\RequirePackage{fix-cm}
\documentclass[twocolumn,natbib,envcountsame]{svjour3}
\smartqed
\usepackage{url}
\usepackage{graphicx}
\usepackage{mathptmx}
\usepackage[T1]{fontenc}
\usepackage{nicefrac}
\usepackage{amsmath}
\usepackage{amssymb}
\usepackage[mathscr]{eucal}
\usepackage{braket}
\usepackage[caption=off]{subfig}
\usepackage[ruled, linesnumbered]{algorithm2e}
\usepackage{etoolbox}
\makeatletter
\patchcmd{\algocf@latexcaption}{#3}{#3\endgraf}{}{}
\renewcommand{\algocf@typo}{\KwSty}
\makeatother
\SetAlCapFnt{\small}
\SetAlCapNameFnt{\small}
\SetAlgoCaptionSeparator{}
\usepackage{marvosym}
\newcommand{\envelope}{(\raisebox{-.5pt}{\scalebox{1.45}{\Letter}})}
\defcitealias{RP177-1993}{SMPTE RP~177-1993}
\newcommand{\N}{\mathbb{N}}
\newcommand{\R}{\mathbb{R}}
\newcommand{\X}{\mathcal{X}}
\newcommand{\intervalcc}[2]{\left[#1,\ #2\right]}
\newcommand{\intervaloo}[2]{\left(#1,\ #2\right)}
\newcommand{\intervalco}[2]{\left[#1,\ #2\right)}
\newcommand{\abs}[1]{\left\vert #1 \right\vert}
\newcommand{\norm}[1]{\left\Vert#1\right\Vert}
\newcommand{\snorm}[1]{\Vert#1\Vert}
\newcommand{\discint}[2]{\{#1,\dotsc,#2\}}
\newcommand{\inint}[2]{\in\discint{#1}{#2}}
\newcommand{\transp}{^T}
\newcommand{\comp}{^C}
\newcommand{\lag}{\mathscr{L}}
\DeclareMathOperator{\proj}{proj}
\newcommand{\argmin}{\operatornamewithlimits{arg\,min}}
\newcommand{\scndmax}{\operatorname{2nd-max}}

\newcommand{\booleanvalues}{\set{\operatorname{true},\ \operatorname{false}}}
\DeclareMathOperator{\finished}{finished}
\DeclareMathOperator{\solver}{solver}
\DeclareMathOperator{\bisection}{Bisection}
\DeclareMathOperator{\newton}{Newton}
\DeclareMathOperator{\halley}{Halley}
\DeclareMathOperator{\newtonsqr}{NewtonSqr}
\DeclareMathOperator{\lo}{lo}
\DeclareMathOperator{\up}{up}
\DeclareMathOperator{\aux}{auxiliary}
\newcommand{\sy}{\mathrm{sum}_{\tilde{y}}}
\newcommand{\scpy}{\mathrm{scp}_{\tilde{p},\,\tilde{y}}}
\DeclareMathOperator{\EZDL}{EZDL}
\DeclareMathOperator{\vect}{vec}
\DeclareMathOperator{\new}{new}
\DeclareMathOperator{\picksample}{pick\_sample}
\DeclareMathOperator{\epoch}{epoch}
\usepackage{enumitem}
\setlist[enumerate]{label=(\alph*)}
\spnewtheorem*{representationtheorem}{Representation Theorem}{\bfseries}{\itshape}

\begin{document}

\title{Efficient Dictionary Learning with Sparseness-Enforcing Projections}
\author{Markus Thom \and Matthias Rapp \and G\"unther Palm}
\authorrunning{Thom, Rapp, and Palm}
\institute{
Communicated by Julien Mairal, Francis Bach, Michael Elad.\\\\
M. Thom \envelope \at
driveU / Institute of Measurement, Control and Microtechnology\\
Ulm University, 89081 Ulm, Germany\\
\email{markus.thom@uni-ulm.de} \and
M. Rapp \at
driveU / Institute of Measurement, Control and Microtechnology\\
Ulm University, 89081 Ulm, Germany\\
\email{matthias.rapp@uni-ulm.de} \and
G. Palm\\
Institute of Neural Information Processing\\
Ulm University, 89081 Ulm, Germany\\
\email{guenther.palm@uni-ulm.de}}
\date{}

\maketitle

\begin{abstract}
Learning dictionaries suitable for sparse coding instead of using engineered bases has proven effective in a variety of image processing tasks.
This paper studies the optimization of dictionaries on image data where the representation is enforced to be explicitly sparse with respect to a smooth, normalized sparseness measure.
This involves the computation of Euclidean projections onto level sets of the sparseness measure.
While previous algorithms for this optimization problem had at least quasi-linear time complexity, here the first algorithm with linear time complexity and constant space complexity is proposed.
The key for this is the mathematically rigorous derivation of a characterization of the projection's result based on a soft-shrinkage function.
This theory is applied in an original algorithm called Easy Dictionary Learning (EZDL), which learns dictionaries with a simple and fast-to-compute Hebbian-like learning rule.
The new algorithm is efficient, expressive and particularly simple to implement.
It is demonstrated that despite its simplicity, the proposed learning algorithm is able to generate a rich variety of dictionaries, in particular a topographic organization of atoms or separable atoms.
Further, the dictionaries are as expressive as those of benchmark learning algorithms in terms of the reproduction quality on entire images, and result in an equivalent denoising performance.
EZDL learns approximately $30\;\%$ faster than the already very efficient Online Dictionary Learning algorithm, and is therefore eligible for rapid data set analysis and problems with vast quantities of learning samples.
\end{abstract}

\section{Introduction}
In a great variety of classical machine learning problems, sparse solutions are attractive because they provide more efficient representations compared to non-sparse solutions.
There is an overwhelming evidence that mammalian brains respect the sparseness principle \citep{Laughlin2003}, which holds true especially for the mammalian visual cortex \citep{Hubel1959,Olshausen1996,Olshausen1997}.
It suggests itself that sparseness be a fundamental prior to a variety of signal processing tasks.
In particular, this includes low-level image processing since natural images can be represented succinctly using structural primitives \citep{Olshausen1997,Mairal2009}.
Interesting and biologically plausible sparse representations were discovered through computer simulations on natural images \citep{Olshausen1996,Olshausen1997}.
Related representations can be obtained by analysis of temporal image sequences \citep{Hateren1998a,Olshausen2003}, stereo image pairs and images with chromatic information \citep{Hoyer2000}, or by enforcing a topographic organization \citep{Hyvarinen2001,Kavukcuoglu2009}.

Sparseness alleviates the effects of random noise in a natural way since it prevents arbitrary combinations of measured signals \citep{Donoho1995,Hyvarinen1999,Elad2006}.
In fact, methods based on sparse representations were shown to achieve state-of-the-art performance for image denoising \citep{Mairal2009}.
Further notable image processing applications that benefit from the efficiency gained through sparseness are as diverse as deblurring \citep{Dong2011}, super-resolution \citep{Yang2010,Yang2012,Dong2011}, compression \citep{Skretting2011,Horev2012}, and depth estimation \citep{Tosic2011}.

\subsection{Dictionary Learning and Sparseness Measures}
Each of these tasks needs a model capable of reproducing the signals to be processed.
In a \emph{linear generative model for sparse coding,} a sample $x\in\R^d$ with $d$ features should be expressed approximately as a linear combination of only a few atoms of a larger dictionary:
\begin{displaymath}
  x\approx W h\text{, where }W\in\R^{d\times n}\text{ and }h\in\R^n\text{ is sparsely populated.}
\end{displaymath}
Here, $W$ is the \emph{dictionary} which is fixed for all samples, and $h$ is a \emph{sparse code word} that depends on the concrete sample.
The $n$ columns of $W$ represent the \emph{atoms}, which are also called bases or filters.
This sparse coding framework is well-suited for overcomplete representations where $n \gg d$.
The dictionary can be generated by wavelets, for example, or adapted to a specific task by solving an optimization problem on measurement data.
The latter is also called \emph{dictionary learning} in this context.

Sparseness acts as a regularizer.
If $h$ was not constrained to sparseness, then trivial choices of $W$ would suffice for perfect reproduction capabilities.
But when $h$ is sparse, then $x$ can be represented by additive superposition of only a small number of bases, thus preventing trivial solutions.

A fundamental problem when working with sparse representations is how to decide on a function that formally assesses the sparseness of a vector.
The $L_0$ pseudo-norm
\begin{displaymath}
  \norm{\cdot}_0\colon\R^n\to\discint{0}{n}\text{,}\quad x\mapsto\abs{\set{i\inint{1}{n} | x_i\neq 0}}\text{,}
\end{displaymath}
simply counts the nonzero entries of its argument.
It is a poor choice since it is non-continuous, prone to random noise and fails to fulfill desirable properties of meaningful sparseness measures \citep{Hurley2009}.

\begin{figure}[t]
  \centering
  \includegraphics[width=67mm]{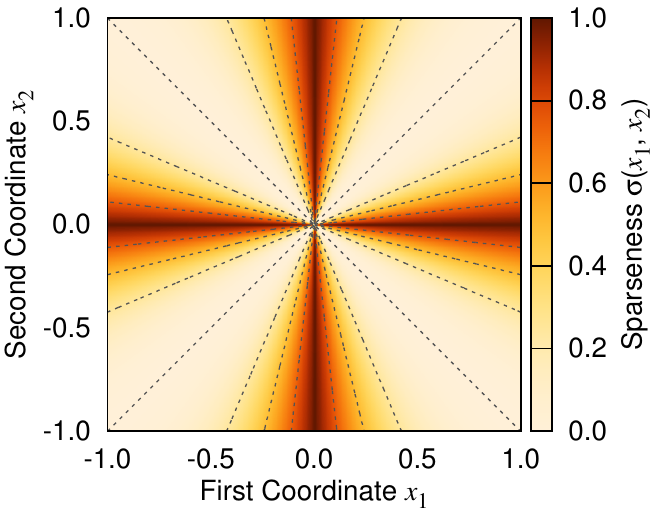}
  \caption{Visualization of Hoyer's sparseness measure $\sigma$. The abscissa and the ordinate specify the entries of a two-dimensional vector, the obtained sparseness degree is color coded. The dashed lines are contour levels at intervals of $0.25$.}
  \label{fig:sparseness_measure_hoyer}
\end{figure}

Throughout this paper, we will use the smooth, normalized sparseness measure $\sigma$ proposed by \citet{Hoyer2004}:
\begin{displaymath}
  \sigma\colon\R^n\setminus\set{0}\to\intervalcc{0}{1}\text{,}\quad x\mapsto\frac{\sqrt{n} - \nicefrac{\norm{x}_1}{\norm{x}_2}}{\sqrt{n}-1}\text{.}
\end{displaymath}
Here, $\norm{\cdot}_1$ and $\norm{\cdot}_2$ denote the Manhattan norm and the Euclidean norm, respectively.
The normalization has been designed such that $\sigma$ attains values between zero and one.
When $x\in\R^n$ satisfies $\sigma(x) = 1$, then all entries of $x$ but one vanish.
Conversely, when $\sigma(x) = 0$, then all the entries of $x$ are equal.
The function $\sigma$ interpolates smoothly between these extremes, see Fig.~\ref{fig:sparseness_measure_hoyer}.
Moreover, it is scale-invariant so that the same sparseness degree is obtained when a vector is multiplied with a nonzero number.
Hence if a quantity is given in other units, for example in millivolts instead of volts, no adjustments whatsoever have to be made.

The sparseness degree with respect to $\sigma$ does not change much if a small amount is added to all entries of a vector, whereas the $L_0$ pseudo-norm would indicate that the new vector is completely non-sparse.
These properties render Hoyer's sparseness measure intuitive, especially for non-experts.
It has been employed successfully for dictionary learning \citep{Hoyer2004,Potluru2013}, and its smoothness results in improved generalization capabilities in classification tasks compared to when the $L_0$ pseudo-norm is used \citep{Thom2013}.

\subsection{Explicit Sparseness Constraints and Projections}
A common approach to dictionary learning is the minimization of the reproduction error between the original samples from a learning set and their approximations provided by a linear generative model under sparseness constraints \citep{Olshausen1996,Olshausen1997,Kreutz-Delgado2003,Mairal2009a}.
It is beneficial for practitioners and end-users to enforce \emph{explicit} sparseness constraints by demanding that all the code words $h$ in a generative model possess a target sparseness degree of $\sigma_H\in\intervaloo{0}{1}$.
This leads to optimization problems of the form
\begin{displaymath}
  \min_{W,\,h}\ \norm{x - Wh}_2^2\text{ so that }\sigma(h) = \sigma_H\text{.}
\end{displaymath}
Here, the objective function is the reproduction error implemented as Euclidean distance.
\emph{Implicit} sparseness constraints, on the other hand, augment the reproduction error with an additive penalty term, yielding optimization problems such as
\begin{displaymath}
  \min_{W,\,h}\ \norm{x - Wh}_2^2 + \lambda\norm{h}_1\text{.}
\end{displaymath}
Here, $\lambda > 0$ is a trade-off constant and the Manhattan norm is used to penalize non-sparse code words as convex relaxation of the $L_0$ pseudo-norm \citep{Donoho2006a}.

Trading off the reproduction error against an additive sparseness penalty is non-trivial since the actual resulting code word sparseness cannot easily be predicted.
Explicit constraints guarantee that the adjusted sparseness degree is met, making tuning of intransparent scale factors such as $\lambda$ in the example above obsolete.
This way, one can concentrate on the actual application of the theory rather than having to develop an intuition of the meaning of each and every parameter.

The mathematical tool to achieve explicit sparseness is a \emph{sparseness-enforcing projection operator}.
This is a vector-valued function which maps any given point in Euclidean space to its nearest point that achieves a pre-set target sparseness degree.
One use case of this theory is projected gradient methods \citep{Bertsekas1999}, where a given objective function should be optimized subject to hard side conditions. 
Replacing the parameters with their best approximations lying in a certain set after each update step ensures that the constraints are satisfied during optimization progress.

\subsection{Contributions of this Paper and Related Work}
This paper studies dictionary learning under explicit sparseness constraints with respect to Hoyer's sparseness measure $\sigma$.
A major part of this work is devoted to the efficient algorithmic computation of the sparseness-enforcing projection operator, which is an integral part in efficient dictionary learning.
Several algorithms were proposed in the past to solve the projection problem \citep{Hoyer2004,Theis2005,Potluru2013,Thom2013}.
Only \citet{Thom2013} provided a complete and mathematically satisfactory proof of correctness for their algorithm.
Moreover, all known algorithms have at least quasi-linear time complexity in the dimensionality of the vector that should be projected.

In this paper, we first derive a characterization of the sparseness projection and demonstrate that its computation is equivalent to finding the root of a real-valued auxiliary function, which constitutes a much simpler problem.
This result is used in the proposition of an algorithm for the projection operator that is asymptotically optimal in the sense of complexity theory, that is the time complexity is linear and the space complexity is constant in the problem dimensionality.
We show through experiments that when run on a real computing machine the newly proposed algorithm is far superior in its computational demands to previously proposed techniques, even for small problem dimensionalities.

Existing approaches to dictionary learning that feature explicit sparseness constraints can be categorized into ones that use Hoyer's $\sigma$ and ones that employ the $L_0$ pseudo-norm.
\citet{Hoyer2004} and \citet{Potluru2013} considered matrix factorization frameworks subject to $\sigma$ constraints, with space requirements linear in the number of learning samples which prevents processing large data sets.
\citet{Thom2013} designed sparse code words as the result of the sparseness-enforcing projection operator applied to the product of the dictionary with the samples.
This requires the computation of the projection's gradient during learning, which is feasible yet difficult to implement and has non-negligible adverse effects on the execution time.

The following approaches consider explicit $L_0$ pseudo-norm constraints:
\citet{Aharon2006}, \citet{Skretting2010} and \citet{Coates2011} infer sparse code words in each iteration compatible to the data and dictionary by employing basis pursuit or matching pursuit algorithms, which has a negative impact on the processing time.
\citet{Zelnik2012} consider block-sparse representations, here the signals are assumed to reside in the union of few subspaces.
\citet{Duarte2009} propose to simultaneously learn the dictionary and the sensing matrix from example image data, which results in improved reconstruction results in compressed sensing scenarios.

In addition to the contributions on sparseness projection computation, this paper proposes the \emph{Easy Dictionary Learning} (EZDL) algorithm.
Our technique aims at dictionary learning under explicit sparseness constraints in terms of Hoyer's sparseness measure $\sigma$ using a simple, fast-to-compute and biologically plausible Hebbian-like learning rule.
For each presented learning sample, the sparseness-enforcing projection operator has to be carried out.
The ability to perform projections efficiently makes the proposed learning algorithm particularly efficient: $30\;\%$ less training time is required in comparison to the optimized Online Dictionary Learning method of \citet{Mairal2009a}.

Extensions of Easy Dictionary Learning facilitate alternative representations such as topographic atom organization or atom sparseness, which includes for example separable filters.
We furthermore demonstrate the competitiveness of the dictionaries learned with our algorithm with those computed with alternative sophisticated dictionary learning algorithms in terms of reproduction and denoising quality of natural images.
Since other tasks, such as deblurring or super-resolution, build upon the same optimization problem in the application phase as reproduction and denoising, it can be expected that EZDL dictionaries will exhibit no performance degradations in those tasks either.

The remainder of this paper is structured as follows.
Section~\ref{sect:projfunc} derives a linear time and constant space algorithm for the computation of sparseness-enforcing projections.
In Sect.~\ref{sect:dictlearn}, the Easy Dictionary Learning algorithm for explicitly sparseness-constrained dictionary learning is proposed.
Section~\ref{sect:experimental_results} reports experimental results on the performance of the newly proposed sparseness projection algorithm and the Easy Dictionary Learning algorithm.
Section~\ref{sect:conclusions} concludes the paper with a discussion, and the appendix contains technical details of the mathematical statements from Sect.~\ref{sect:projfunc}.

\section{Efficient Sparseness-Enforcing Projections}
\label{sect:projfunc}
This section proposes a linear time and constant space algorithm for computation of projections onto level sets of Hoyer's $\sigma$.
Formally, if $\sigma^*\in\intervaloo{0}{1}$ denotes a \emph{target sparseness degree} and $x\in\R^n$ is an arbitrary point, the point from the level set $S := \Set{s\in\R^n | \sigma(s) = \sigma^*}$ that minimizes the Euclidean distance to $x$ is sought.
The function that computes $\arg\min_{s\in S}\norm{x - s}_2$ is also called \emph{sparseness-enforcing projection operator} since the situation where $\sigma(x) < \sigma^*$ is of particular interest.

Due to symmetries of $\sigma$, the above described optimization problem can be reduced to finding the projection of a vector $x\in\R_{\geq 0}^n$ with non-negative coordinates onto the set
\begin{displaymath}
  T := \Set{s\in\R_{\geq 0}^n | \norm{s}_1 = \lambda_1\text{ and }\norm{s}_2 = \lambda_2}\subseteq S\text{,}
\end{displaymath}
where $\lambda_1$ and $\lambda_2$ are target norms that should be chosen such that $\sigma^* = \left(\sqrt{n} - \nicefrac{\lambda_1}{\lambda_2}\right) / \left(\sqrt{n}-1\right)$ \citep{Hoyer2004,Thom2013}.
With this choice, $\sigma$ constantly attains the value of $\sigma^*$ on the entire set $T$, hence $T$ is a subset of $S$.

\citet{Thom2013} proved that $T$ is the intersection of a scaled canonical simplex and a hypercircle.
They further demonstrated that projections onto $T$ can be computed with a finite number of alternating projections onto these two geometric structures.
The proof of correctness of this method could not rely on the classical alternating projection method \citep[see for example][]{Deutsch2001} due to the lacking convexity of $T$, rendering the proof arguments quite complicated.

The remainder of this section proceeds as follows.
First, a succinct characterization of the sparseness projection is given.
It is then shown that computing sparseness projections is equivalent to finding the zero of a monotone real-valued function.
We prove that this can be achieved with optimal asymptotic complexity by proposing a new algorithm that solves the projection problem.

\subsection{Representation of the Projection}
The main theoretical result of this paper is a representation theorem that characterizes the projection onto $T$.
This was gained through an analysis of the intermediate points that emerge from the alternating projections onto a simplex and a hypercircle.
A closed-form expression can then be provided by showing that the intermediate points in the projection algorithm satisfy a loop-invariant, a certain property fulfilled after each step of alternating projections.

Since a mathematical rigorous treatment of this result is very technical, it is deferred to the appendix.
We assume here that the vector $x$ to be projected is given so that its projection is unique.
Since it is guaranteed that for all points except for a null set there is exactly one projection \citep[Theorem~2.6]{Theis2005}, we can exclude points with non-unique projections in our considerations, which is no restriction in practice.
We may now state a characterization of the projection outcome:
\begin{representationtheorem}
Let $x\in\R_{\geq 0}^n$ and let $p\in T$ denote the unique projection of $x$ onto $T$.
Then there is exactly one real number $\alpha^*$ with
\begin{displaymath}
  p = \frac{\lambda_2}{\norm{q}_2}\cdot q\text{ where }q := \max\left(x - \alpha^*\cdot e,\ 0\right)\in\R^n\text{.}
\end{displaymath}
Here, $e\in\set{1}^n$ is the $n$-dimensional vector where all entries are unity.
If the indices of the positive coordinates of $p$ are known, then $\alpha^*$ can be computed directly from $x$ with an explicit expression.
\end{representationtheorem}

In words, the projection $p$ is the point $q$ rescaled to obtain an $L_2$ norm of $\lambda_2$.
The vector $q$ is computed from the input vector $x$ by subtracting the scalar $\alpha^*$ from all the entries, and afterwards setting negative values to zero.
It is remarkable that the projection admits such a simple representation although the target set for the projection is non-convex and geometrically quite complicated.

The function that maps a scalar $\xi\in\R$ to $\max(\xi - t,\ 0)$ for a constant offset $t\in\R$ is called the \emph{soft-shrinkage function}.
If $t = 0$, it is also called the \emph{rectifier function}.
Because the central element of the projection representation is a soft-shrinkage operation applied entry-wise to the input vector, carrying out projections onto level sets of Hoyer's sparseness measure can be interpreted as denoising operation \citep{Donoho1995,Hyvarinen1999,Elad2006}.

\subsection{Auxiliary Function for the Sparseness Projection}
\label{sect:auxfunction}
The projection problem can hence be reduced to determining the soft-shrinkage offset $\alpha^*$, which is a one-dimensional problem on the real line.
Further, it is reasonable that $\alpha^*$ must be smaller than the maximum entry of $x$ since otherwise $q$ would be the null vector, which is absurd.
In the usual case where $\sigma(x) < \sigma^*$ we can further conclude that $\alpha^*$ must be non-negative.
Otherwise, the result of the projection using the representation theorem would be less sparse than the input, which is impossible.
Therefore, the projection problem can be further reduced to finding a number from the bounded interval $\intervalco{0}{x_{\max}}$ with $x_{\max}:= \max_{i\inint{1}{n}}x_i$.

Next, a method for efficiently deciding whether the correct offset has been found is required.
Similar to the projection onto a canonical simplex \citep{Liu2009a}, we can design a real-valued function that vanishes exactly at the wanted offset.
The properties of Hoyer's $\sigma$ allow us to formulate this function in an intuitive way.
We call
\begin{displaymath}
  \Psi\colon\intervalco{0}{x_{\max}}\to\R\text{,}\quad \alpha\mapsto\frac{\norm{\max\left(x - \alpha\cdot e,\ 0\right)}_1}{\norm{\max\left(x - \alpha\cdot e,\ 0\right)}_2} - \frac{\lambda_1}{\lambda_2}\text{,}
\end{displaymath}
the \emph{auxiliary function} for the projection onto $T$.

The rationale for the concrete definition of $\Psi$ is as follows.
Due to the representation theorem we know that the projection $p$ onto $T$ is merely a scaled version of the point $q := \max\left(x - \alpha^*\cdot e,\ 0\right)$.
Moreover $\sigma(p) = \sigma^*$, and due to the scale-invariance of Hoyer's $\sigma$ follows $\sigma(q) = \sigma^*$.
The essence of $\sigma$ is the ratio of the $L_1$ norm to the $L_2$ norm of its argument.
Hence here the scale constants that make $\sigma$ normalized to the interval $\intervalcc{0}{1}$ are omitted and the target norms are used instead.
We therefore have that $\nicefrac{\norm{q}_1}{\norm{q}_2} = \nicefrac{\lambda_1}{\lambda_2}$, and thus $\Psi(\alpha^*) = 0$.
As $\alpha^*$ is unique, we can conclude that no other offset of the soft-shrinkage operation leads to a correct solution.
In fact, if $\alpha\neq\alpha^*$ and when we write $\tilde{q} := \max\left(x - \alpha\cdot e,\ 0\right)$, then we have that $\sigma(\tilde{q})\neq\sigma^*$ and therefore $\Psi(\alpha)\neq 0$.

\begin{figure}[t]
  \centering
  \includegraphics[width=84mm]{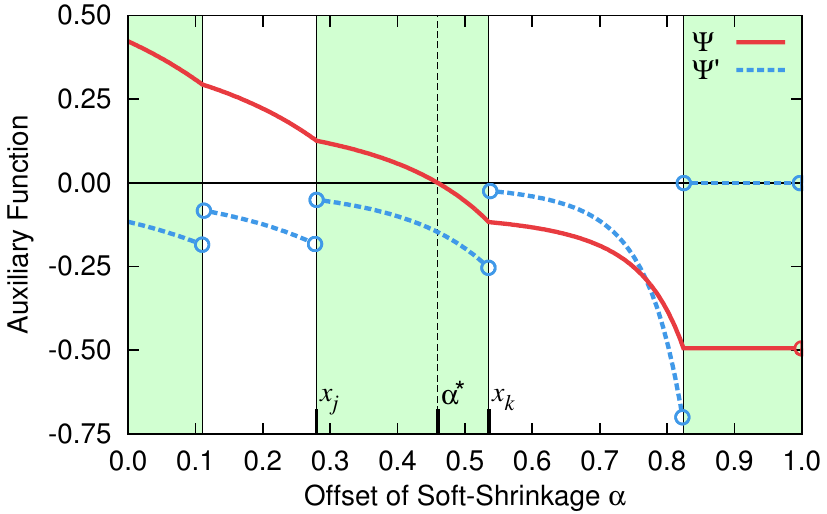}
  \caption{Plot of the auxiliary function $\Psi$ and its derivative for a random vector $x$. The derivative $\Psi'$ was scaled using a positive number for improved visibility. The steps in $\Psi'$ are exactly the places where $\alpha$ coincides with an entry of $x$. It is enough to find an $\alpha$ with $\Psi(x_j) \geq 0$ and $\Psi(x_k) < 0$ for the neighboring entries $x_j$ and $x_k$ in $x$, because then the exact solution $\alpha^*$ can be computed with a closed-form expression.}
  \label{fig:projfunclin-toy}
\end{figure}

An exemplary plot of $\Psi$ is depicted in Fig.~\ref{fig:projfunclin-toy}.
Clearly, the auxiliary function is continuous, differentiable except for isolated points and strictly decreasing except for its final part where it is constant.
Since the constant part is always negative and starts at the offset equal to the second-largest entry $x_{\scndmax} := \max\set{x_i | i\inint{1}{n}\text{ and }x_i\neq x_{\max}}$ of $x$, the feasible interval can be reduced even more.
The step discontinuities in $\Psi'$ coincide exactly with the entries of the input vector $x$.
These appealing analytical properties greatly simplify computation of the zero of $\Psi$, since standard root-finding algorithms such as bisection or Newton's method \citep[see for example][]{Traub1964} can be employed to numerically find $\alpha^*$.

\subsection{Linear Time and Constant Space Projection Algorithm}
\label{sect:lintimeconstspacealg}
We can improve on merely a numerical solution by exploiting the special structure of $\Psi$ to yield an analytical solution to the projection problem.
This is due to the closed-form expression for $\alpha^*$ which requires the indices of the coordinates in which the result of the projection is positive to be known.
The other way around, when $\alpha^*$ is known, these indices are exactly the ones of the coordinates in which $x$ is greater than $\alpha^*$.
When an offset sufficiently close to $\alpha^*$ can be determined numerically, the appropriate index set can be determined and thus the exact value of $\alpha^*$.

The decision if a candidate offset $\alpha$ is close enough to the true $\alpha^*$ can be made quite efficiently, since $\alpha^*$ is coupled via index sets to individual entries of the input vector $x$.
Hence, it suffices to find the right-most left neighbor $x_j$ of $\alpha$ in the entries of $x$, and analogously the left-most right neighbor $x_k$ (see Fig.~\ref{fig:projfunclin-toy} for an example).
Whenever $\Psi(x_j) \geq 0$ and $\Psi(x_k) < 0$, the zero $\alpha^*$ of $\Psi$ must be located between $x_j$ and $x_k$ for continuity reasons.
But then all values in $x$ greater than $\alpha^*$ are exactly the values greater than or equal to $x_k$, which can be determined by simply scanning through the entries of $x$.

\begin{figure}[t]
  \centering
  \includegraphics{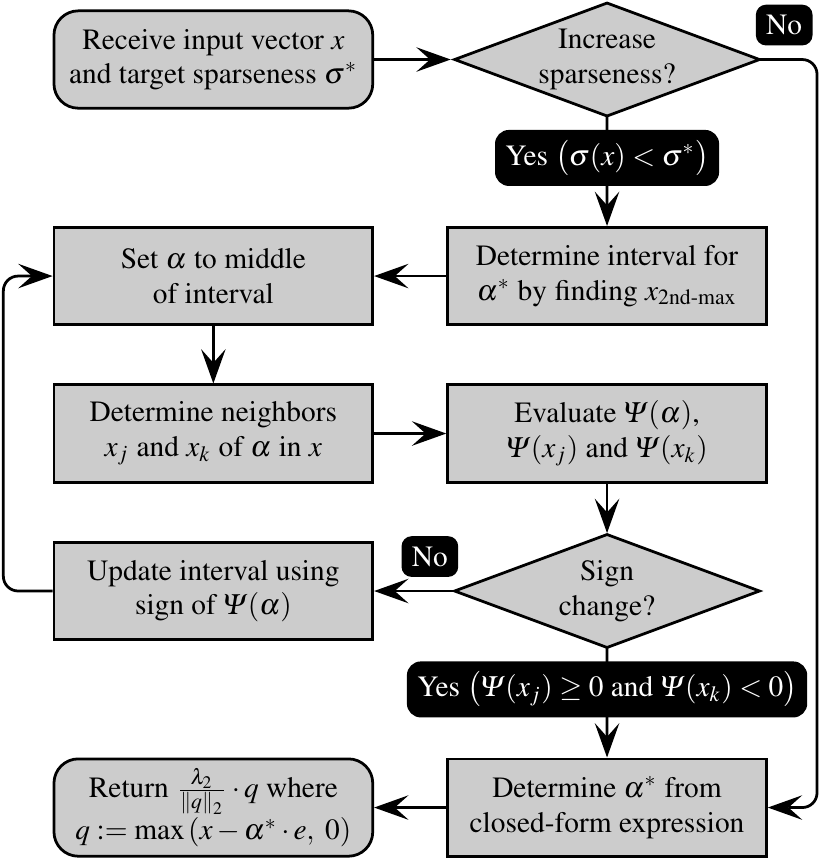}
  \caption{Flowchart of the proposed algorithm for computing sparseness-enforcing projections. The algorithm starts with the box at the upper left and terminates with the box at the lower left.}
  \label{fig:projfunc-flowchart}
\end{figure}

Based upon these considerations, a flowchart of our proposed method for computing sparseness-enforcing projections is depicted in Fig.~\ref{fig:projfunc-flowchart}.
The algorithm performs bisection and continuously checks for sign changes in the auxiliary function $\Psi$.
As soon as this is fulfilled, $\alpha^*$ is computed and the result of the projection is determined in-place.
A complete formal presentation and discussion of our proposed algorithm is given in the appendix.

Each intermediate step of the algorithm (that is, determination of the second-largest entry of $x$, evaluation of the auxiliary function, projection result computation by application of soft-shrinkage and scaling) can be carried out in a time complexity linear in the problem dimensionality $n$ and a space complexity completely independent of $n$.

Therefore, to show the overall algorithm possesses the same asymptotic complexity, it has to be proved that the number of bisection iterations required for finding $\alpha^*$ is independent of $n$.
The length of the initial interval is upper-bounded by $x_{\scndmax}$ as discussed earlier.
Bisection is moreover guaranteed to stop at the latest once the interval length is smaller than the minimum pairwise distance between distinct entries of $x$,
\begin{displaymath}
  \delta := \min\Set{x_k - x_j | x_j,x_k\in\X\text{ and }x_j < x_k} > 0
\end{displaymath}
where $\X := \set{x_i | i\inint{1}{n}}$.
This is due to the fact that it is enough to find a sufficiently small range to deduce an analytical solution.
The required number of bisection iterations is then less than $\lceil\log_2({x_{\scndmax}} / {\delta})\rceil$, see \citet{Liu2009a}.
This number is bounded from above regardless of the dimensionality of the input vector since $x_{\scndmax}$ is upper-bounded and $\delta$ is lower-bounded due to finite machine precision \citep{Goldberg1991}.

Hence our sparseness-enforcing projection algorithm is asymptotically optimal in the sense of complexity theory.
There may still be hidden constants in the asymptotic notation that render the proposed algorithm less efficient than previously known algorithms for a small input dimensionality.
Experiments in Sect.~\ref{sect:experimental_results} demonstrate that this is not the case.

\section{Explicitly Sparseness-Constrained Dictionary~Learning}
\label{sect:dictlearn}
This section proposes the \emph{Easy Dictionary Learning} (EZDL) algorithm for dictionary learning under explicit sparseness constraints.
First, we introduce an ordinary formulation of the learning algorithm.
Sparse code word inference with the sparseness projection algorithm proposed in Sect.~\ref{sect:projfunc} renders EZDL efficient and particularly simple to implement.
We further discuss extensions that allow concentrating on different aspects of the data set under investigation, such as topographic organization of the atoms and atom sparseness.
Only little implementation effort is required for these extensions and their computational demands are low.
A description of the comprehensive EZDL learning algorithm is accompanied with several strategies that improve the optimization performance.

\subsection{EZDL---Easy Dictionary Learning}
Conventional dictionary learning algorithms operate in an alternating fashion, where code words are inferred with a costly optimization procedure before updating the dictionary.
EZDL learns a dictionary by first yielding sparse code words through simple inference models and then tuning the dictionary with a simple update step.
An \emph{inference model} is a function $I\colon\R^n\to\R^n$ which accepts \emph{filter responses} in the form of the product of the dictionary $W\in\R^{d\times n}$ with a learning sample $x\in\R^d$, $u := W\transp x\in\R^n$, and produces a representation $h := I(u)$ with certain desirable properties.
Here, the combination with the sparseness-enforcing projection operator is particularly interesting since it provides a natural method to fulfill explicit sparseness constraints.

We call the choice $I(u) := \Pi_{\sigma_H}(u)$ the \emph{ordinary inference model}, where $\Pi_{\sigma_H}$ denotes the Euclidean projection onto the set of all vectors that achieve a sparseness degree of $\sigma_H\in\intervaloo{0}{1}$ with respect to Hoyer's sparseness measure $\sigma$.
This computation can also be interpreted as the trivial first iteration of a projected Landweber procedure for sparse code word inference \citep{Bredies2008}.

The dictionary $W$ is adapted to new data by minimizing the deviation between a learning sample $x$ and its approximation $Wh$ through the linear generative model.
The goodness of the approximation is here assessed with a differentiable similarity measure $\rho\colon\R^d\times\R^d\to\R$, so that the EZDL optimization problem becomes:
\begin{displaymath}
  \min_{W}\ \rho(Wh,\ x)\text{.}
\end{displaymath}
Note that $h$ is not a variable of the optimization problem since it is defined as output of an inference model.
For the same reason, $h$ does not need to be constrained further as it inherently satisfies explicit sparseness constraints.

Although a wide range of similarity measures is possible, we decided to use the (Pearson product-moment) correlation coefficient since it is invariant to affine-linear transformations \citep{Rodgers1988}.
In the context of visual data set analysis, this corresponds to invariance with respect to DC components and gain factors.
Moreover, differentiability of this similarity measure facilitates gradient-based learning \citep{Thom2013}.

If $\rho'\colon\R^d\times\R^d\to\R^d$ is the transposed derivative of $\rho$ with respect to its first argument and $g := \rho'(Wh,\ x)\in\R^d$ denotes its value in $Wh$, the gradient of $E_{\EZDL}$ for the dictionary is given by
\begin{displaymath}
  \left(\frac{\partial E_{\EZDL}}{\partial W}\right)\transp = gh\transp\in\R^{d\times n}\text{.}
\end{displaymath}
Therefore, when $W$ is tuned with gradient descent a simple and biologically plausible Hebbian-like learning rule \citep{Hyvaerinen2009} results:
\begin{displaymath}
  W^{\new} := W - \eta\cdot g h\transp\text{, where }\eta > 0\text{ denotes the step size.}
\end{displaymath}
This update step can be implemented efficiently with simple rank-1 updates \citep{Blackford2002}.

We here assumed $h = I(W\transp x)$ to be constant and neglected that it actually depends on $W$.
Without this assumption, the gradient would have to be extended with an additive term which comprises the gradient of the inference model \citep[Proposition~39]{Thom2013}.
This, in turn, requires the computation of the sparseness-enforcing projection operator's gradient.
This gradient can be computed with simple vector operations as we show in the appendix.
However, this constitutes a significant computational burden compared to the simple and fast-to-compute EZDL update step.
Section~\ref{sect:experimental_results} demonstrates through experiments that this simplification is still perfectly capable of learning dictionaries.

\subsection{Topographic Atom Organization and Atom Sparseness}
Topographic organization of the dictionary's atoms similar to Self-organizing Maps \citep{Kohonen1990} or Topographic Independent Component Analysis \citep{Hyvarinen2001} can be achieved in a straightforward way with EZDL using alternative inference models proposed in this section.
For this, the dictionary atoms are interpreted to be arranged on a two-dimensional grid.
A spatial pooling operator subject to circular boundary conditions can then be used to incorporate interactions between atoms located adjacent on the grid.
This can, for example, be realized by averaging each entry of the filter responses in a $3\times 3$ neighborhood on the grid.
This convolution operation can be expressed as linear operator applied to the vector $u\in\R^n$, represented by a sparsely populated matrix $G\in\R^{n\times n}$.
With $G$ containing all atom interactions, we call $I(u) := \Pi_{\sigma_H}(Gu)$ the \emph{topographic inference model}.

An alternative realization of topographic organization is enforcing \emph{structure sparseness} on the filter responses reshaped as matrix to account for the grid layout.
Structure sparseness can be assessed by application of a sparseness measure to the vector with the singular values of a matrix.
When the $L_0$ pseudo-norm is used, then this is the rank of a matrix.
The matrix rank can also be measured more robustly through the ratio of the Schatten $1$-norm to the Schatten $2$-norm, which is essentially Hoyer's sparseness measure $\sigma$ applied to the singular values \citep{Lopes2013}.

Clearly, any matrix $M\in\R^{a\times b}$ where $a,b\in\N$ can be reshaped to a vector $m:= \vect(M)\in\R^{ab}$ with the vectorization operator that stacks all columns of $M$ on top of another \citep{Neudecker1969}.
This linear operation can be inverted provided the shape of the original matrix is known, $M = \vect_{a\times b}^{-1}(m)$.
In the following, let $\Pi_{\kappa^*}$ denote the projection onto the set of all matrices that possess a target rank $\kappa^*\inint{1}{\min\set{a,b}}$.
A classical result states that any matrix can be projected onto the set of low-rank matrices by computation of its Singular Value Decomposition, applying an $L_0$ pseudo-norm projection to the vector of singular values thus retaining only the largest ones, and finally computing the reconstruction using the modified singular values \citep{Eckart1936}.

Based on these considerations, let $r$ and $c$ with $n = rc$ be natural numbers describing the topographic grid layout dimensionalities.
The \emph{rank-$\kappa_H$ topographic inference model} is then defined by the composition
\begin{displaymath}
  I(u) := \left(\vect\circ\Pi_{\kappa_H}\!\circ\vect_{r\times c}^{-1}\circ\Pi_{\sigma_H}\!\right)\!(u)\text{.}
\end{displaymath}
In words, this inference model operates as follows.
It first approximates the filter responses $u$ with a sparsely populated vector that attains a sparseness of $\sigma_H$ with respect to Hoyer's $\sigma$.
These sparsified filter responses are then laid out on a grid and replaced with those best approximating filter responses that meet a lower rank of $\kappa_H\inint{1}{\min\set{r,c}}$.
Finally, the grid layout is reshaped to a vector since this should be the output type of each inference model.

If $\kappa_H = 1$, then there are two vectors $y\in\R^r$ and $z\in\R^c$ such that $\vect_{r\times c}^{-1}\left(I(u)\right) = yz\transp$, that is the filter responses can be expressed as a dyadic product when reshaped to a grid.
We will demonstrate with experiments that this results in an interesting topographic atom organization similar to Independent Subspace Analysis \citep{Hyvaerinen2000a}.

Analogous to Non-negative Matrix Factorization with Sparseness Constraints \citep{Hoyer2004}, the atoms can be enforced to be sparse as well.
To achieve this, it is sufficient to apply the sparseness projection to each column of the dictionary after each learning epoch:
\begin{displaymath}
  W^{\new}e_i := \Pi_{\sigma_W}(We_i)\text{ for all }i\inint{1}{n}\text{,}
\end{displaymath}
where $e_i\in\R^n$ is the $i$-th canonical basis vector that selects the $i$-th column from $W$, and $\sigma_W\in\intervaloo{0}{1}$ is the target degree of atom sparseness.

In the situation where the learning samples resemble image patches, the atoms can also be restricted to fulfill structure sparseness using a low-rank approximation.
Suppose the image patches possess $p_h\times p_w$ pixels, then the following projection can be carried out after each learning epoch:
\begin{displaymath}
  W^{\new}e_i := \left(\vect\circ\Pi_{\kappa_W}\!\circ\vect_{p_h\times p_w}^{-1}\!\right)\!(We_i)\text{ for all }i\inint{1}{n}\text{.}
\end{displaymath}
Here, $\kappa_W\inint{1}{\min\set{p_h,p_w}}$ denotes the target rank of each atom when reshaped to $p_h\times p_w$ pixels.
When the dictionary is used to process images through convolution, small values of $\kappa_W$ are beneficial since computations can then be speeded up considerably.
For example, if $\kappa_W = 1$ each atom can be expressed as outer product of vectors from $\R^{p_h}$ and $\R^{p_w}$, respectively.
The convolutional kernel is separable in this case, and two one-dimensional convolutions lead to the same result as one two-dimensional convolution, but require only a fraction of operations \citep{Hoggar2006}.

\subsection{Learning Algorithm}
In the previous sections, a simple Hebbian-like learning rule was derived which depends on abstract inference models.
The core of the inference models is the sparseness-enforcing projection operator discussed in Sect.~\ref{sect:projfunc}, which guarantees that the code words always attain a sparseness degree easily controllable by the user.
On presentation of a learning sample $x$, an update step of the dictionary $W$ hence consists of
\begin{enumerate}
  \item determining the filter responses ($u := W\transp x$),
  \item feeding the filter responses through the inference model involving a sparseness projection ($h := I(u)$),
  \item computation of the approximation to the original learning sample ($Wh$) using the linear generative model,
  \item assessing the similarity between the input sample and its approximation ($\rho(Wh,\ x)$) and the similarity measure's gradient ($g := \rho'(Wh,\ x)$), and eventually
  \item adapting the current dictionary using a rank-1 update ($ W^{\new} := W - \eta\cdot g h\transp$).
\end{enumerate}
All these intermediate steps can be carried out efficiently.
The sparseness projection can be computed using few resources, and neither its gradient nor the gradient of the entire inference model is required.
After each learning epoch it is possible to impose additional constraints on the atoms with atom-wise projections.
In doing so, several aspects of the data set structure can be made more explicit since the atoms are then forced to become more simple feature detectors.

An example of an individual learning epoch is given in Algorithm~\ref{alg:ezdl}, where the topographic inference model and low-rank structure sparseness are used.
In addition, the explicit expressions for the correlation coefficient and its gradient are given \citep[Proposition~40]{Thom2013}.
The outcome of this parameterization applied to patches from natural images is depicted in Fig.~\ref{fig:dictionaries_rank1filters_topographic} and will be discussed in Sect.~\ref{sect:rank1_ezdl_experiments}.

\begin{algorithm}[t]
  \caption{One epoch of the Easy Dictionary Learning algorithm, illustrated for the case of a topographic inference model and low-rank structure sparseness.}
  \label{alg:ezdl}
  \SetAlgoLined
  \KwIn{%
The algorithm accepts the following parameters:

\parbox{\hsize}{\raggedright
\begin{itemize}
  \item An existing dictionary $W\in\R^{d\times n}$ with $n$ atoms.
  \item Random access to a learning set with samples from $\R^d$ through a function "$\picksample$".
  \item Step size for update steps $\eta\in\R_{> 0}$.
  \item Number of samples to be presented $M_{\epoch}\in\N$.
  \item For the topographic inference model:
  \begin{itemize}
    \item Target degree of dictionary sparseness $\sigma_H\in\intervaloo{0}{1}$.
    \item Topography matrix $G\in\R^{n\times n}$ describing the interaction between adjacent atoms.
  \end{itemize}
  \item To obtain low-rank structure sparseness:
  \begin{itemize}
    \item Number of pixels $p_h\times p_w$ so that $d = p_h p_w$.
    \item Target rank $\kappa_W\inint{1}{\min\set{p_h,p_w}}$.
  \end{itemize}
\end{itemize}
}%
}
  \KwOut{Updated dictionary $W\in\R^{d\times n}$.}
  \BlankLine

  \tcp{Present samples and adapt dictionary.}
  \SetKwBlock{RepeatMepoch}{repeat $M_{\epoch}$ times}{end}
  \RepeatMepoch
  {%
    \tcp{Pick a random sample.}
    $x := \picksample()\in\R^d$\;
    \tcp{(a) Compute filter responses.}
    $u := W\transp x\in\R^n$\;
    \tcp{(b) Evaluate topographic inference model.}
    $h := \Pi_{\sigma_H}(Gu)\in\R^n$\;
    \tcp{(c) Compute approximation to $x$.}
    $\tilde{x} := Wh\in\R^d$\;
    \BlankLine
    \tcp{(d) Compute correlation coefficient $\rho_{\tilde{x},\;x} := \rho(Wh,\ x)$ and its gradient $g$. Here, $e\in\set{1}^d$ is the all-ones vector.}
    $\lambda := \norm{\tilde{x}}_2^2 - \nicefrac{1}{d}\cdot\left(e\transp\tilde{x}\right)^2\in\R$\;
    $\mu := \norm{x}_2^2 - \nicefrac{1}{d}\cdot\left(e\transp x\right)^2\in\R$\;
    $\rho_{\tilde{x},\;x} := \left(\tilde{x}\transp x - \nicefrac{1}{d}\cdot\left(e\transp\tilde{x}\right)\left(e\transp x\right)\right) / \sqrt{\lambda\mu}\in\R$\;
    $g := \frac{1}{\sqrt{\lambda\mu}}\left(x - \nicefrac{e\transp x}{d}\cdot e\right) - \frac{\rho_{\tilde{x},\;x}}{\lambda}\left(\tilde{x} - \nicefrac{e\transp \tilde{x}}{d}\cdot e\right)\in\R^d$\;
    \BlankLine
    \tcp{(e) Adapt dictionary with rank-1 update.}
    $W := W - \eta\cdot gh\transp\in\R^{d\times n}$\;
  }
  \BlankLine

  \tcp{Atom-wise projection for low-rank filters.}
  \For{$i := 1$ \KwTo $n$}
  {%
    \tcp{Extract $i$-th atom and normalize.}
    $w := \nicefrac{We_i}{\norm{We_i}_2}\in\R^d$\;
    \tcp{Ensure atom has rank $\kappa_W$ when reshaped to $p_h\times p_w$ pixels.}
    $w := \left(\vect\circ\Pi_{\kappa_W}\!\circ\vect_{p_h\times p_w}^{-1}\!\right)\!(w)\in\R^d$\;
    \tcp{Store modified atom back in dictionary.}
    $We_i := w$\;
  }

\end{algorithm}

The efficiency of the learning algorithm can be improved with simple modifications described in the following.
To reduce the learning time compared to when a randomly initialized dictionary was used, the dictionary should be initialized by samples randomly chosen from the learning set \citep{Mairal2009a,Skretting2010,Coates2011,Thom2013}.
When large learning sets with millions of samples are available, stochastic gradient descent has been proven to result in significantly faster optimization progress compared to when the degrees of freedom are updated only once for each learning epoch \citep{Wilson2003,Bottou2004}.
The step size schedule $\eta(\nu) := \nicefrac{\eta_0}{\nu}$, where $\eta_0 > 0$ denotes the initial step size and $\nu\in\N\setminus\set{0}$ is the learning epoch counter, is beneficial when the true gradient is not available but rather an erroneous estimate \citep{Bertsekas1999}.
After each learning epoch, all the atoms of the dictionary are normalized to unit scale.
Since the learning rule is Hebbian-like, this prevents the atoms from growing arbitrarily large or becoming arbitrarily small \citep{Hyvaerinen2009}.
This normalization step is also common in a multitude of alternative dictionary learning algorithms \citep{Kreutz-Delgado2003,Hoyer2004,Mairal2009a,Skretting2010}.

Since the dictionary is modified with a rank-1 update where one of the factors is the result of an inference model and hence sparse, only the atoms that induce significant filter responses are updated.
This may result in atoms that are never updated when the target sparseness degree $\sigma_H$ is large.
This behavior can be alleviated by adding random noise to the inference model's output prior to updating the dictionary, which forces all the atoms to be updated.
We used random numbers sampled from a zero-mean normal distribution, where the standard deviation was multiplicatively annealed after each learning epoch.
As optimization proceeds, the atoms are well-distributed in sample space such that the lifetime sparseness is approximately equal to the population sparseness \citep{Willmore2001}, and randomness is not needed anymore.

\section{Experimental Results}
\label{sect:experimental_results}
This section reports experimental results on the techniques proposed in this paper.
We first evaluate alternative solvers for the sparseness projection's auxiliary function and show that our algorithm for the sparseness-enforcing projection operator is significantly faster than previously known methods.
We then turn to the application of the projection in the Easy Dictionary Learning algorithm.
First, the morphology of dictionaries learned on natural image patches is analyzed and put in context with previous methods.
We further show that the resulting dictionaries are well-suited for the reproduction of entire images.
They achieve a reproduction quality equivalent to dictionaries trained with alternative, significantly slower algorithms.
Eventually, we analyze the performance of the dictionaries when employed for the image denoising task and find that, analogous to the reproduction experiments, no performance degradation can be observed.

\subsection{Performance of Sparseness-Enforcing Projections}
\label{sect:projfuncexperiments}
Our proposed algorithm for sparseness projections was implemented as C++ program using the GNU Scientific Library \citep{Galassi2009}.
We first analyzed whether solvers other than bisection for locating the zero of the auxiliary function would result in an improved performance.
Since $\Psi$ is differentiable except for isolated points and its derivatives can be computed quite efficiently, Newton's method and Halley's method are straightforward to apply.
We further verified whether Newton's method applied to a slightly transformed variant of the auxiliary function,
\begin{displaymath}
  \tilde{\Psi}\colon\intervalco{0}{x_{\max}}\to\R\text{,}\quad \alpha\mapsto\frac{\norm{\max\left(x - \alpha\cdot e,\ 0\right)}_1^2}{\norm{\max\left(x - \alpha\cdot e,\ 0\right)}_2^2} - \frac{\lambda_1^2}{\lambda_2^2}\text{,}
\end{displaymath}
where the minuend and the subtrahend were squared, would behave more efficiently.
The methods based on derivatives were additionally safeguarded with bisection to guarantee new positions are always located within well-defined intervals \citep{Press2007}.
This impairs the theoretical property that only a constant number of steps is required to find a solution, but in practice a significantly smaller number of steps needs to be made compared to plain bisection.

To evaluate which solver would provide maximum efficiency in practical scenarios, one thousand random vectors for each problem dimensionality $n\inint{2^2}{2^{30}}$ were sampled and the corresponding sparseness-enforcing projections for $\sigma^* := 0.90$ were computed using all four solvers.
We have used the very same random vectors as input for all solvers, and counted the number of times the auxiliary function needed to be evaluated until the solution was found.
The results of this experiment are depicted in Fig.~\ref{fig:solverplot}.

\begin{figure}[t]
  \centering
  \includegraphics[width=84mm]{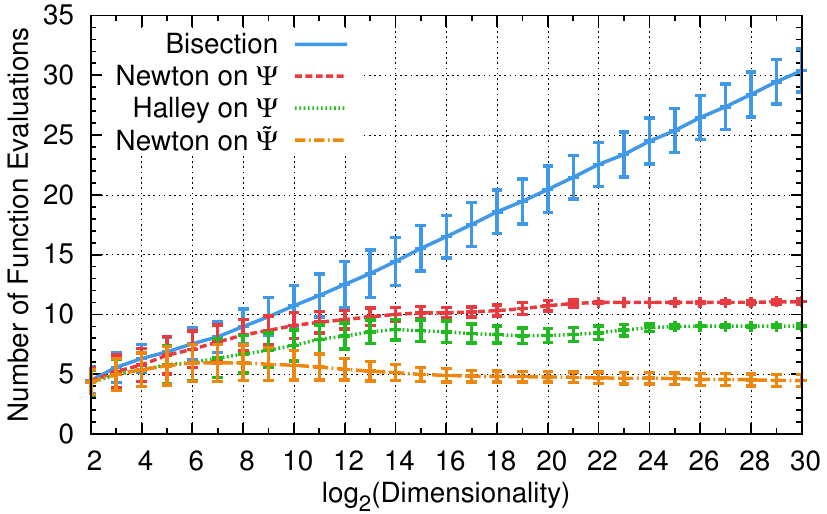}
  \caption{Number of auxiliary function evaluations needed to find the final interval for the projection onto $T$ using four different solvers. The error bars indicate $\pm 1$ standard deviation distance from the mean value. Since Newton's method applied to $\tilde{\Psi}$ is consistently outperforming the other solvers, it is the method of choice for all practical applications.}
  \label{fig:solverplot}
\end{figure}

The number of function evaluations required by plain bisection grew about linearly in $\log(n)$.
Because the expected minimum difference of two distinct entries from a random vector gets smaller when the dimensionality of the random vector is increased, the expected number of function evaluations bisection needs increases with problem dimensionality.
In either case, the length of the interval that has to be found is always bounded from below by the machine precision such that the number of function evaluations with bisection is bounded from above regardless of $n$.

The solvers based on the derivative of $\Psi$ or $\tilde{\Psi}$, respectively, always required less function evaluations than bisection.
They exhibited a growth clearly sublinear in $\log(n)$.
For $n = 2^{30}\approx 10^{9}$, Newton's method required $11$~function evaluations in the mean, Halley's method needed $9$~iterations, and Newton's method applied to $\tilde{\Psi}$ found the solution in only $4.5$~iterations. Therefore, in practice always the latter solver should be employed.

In another experiment, the time competing algorithms require for computing sparseness projections on a real computing machine was measured for a comparison.
For this, the algorithms proposed by \citet{Hoyer2004}, \citet{Potluru2013}, and \citet{Thom2013} were implemented using the GNU Scientific Library \citep{Galassi2009} by means of C++ programs.
An Intel Core i7-4960X processor was used and all algorithms were run in a single-threaded environment.
Random vectors were sampled for each problem dimensionality $n\inint{2^2}{2^{30}}$ and initially projected to attain a sparseness of $0.50$ with respect to Hoyer's $\sigma$.
This initial projection better reflects the situation in practice where not completely random vectors have to be processed.
Next, all four algorithms were used to compute projections with a target sparseness degree of $\sigma^* := 0.90$ and their run time was measured.
The original algorithm of \citet{Hoyer2004} was the slowest, so by taking the ratio of the run times of the other algorithms to the run time of that slowest algorithm the relative speed-up was obtained.

\begin{figure}[t]
  \centering
  \includegraphics[width=84mm]{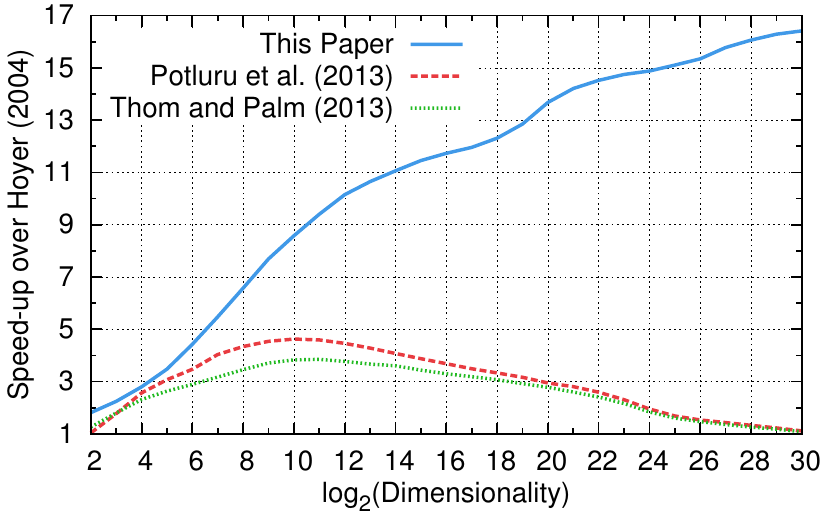}
  \caption{Speed-ups relative to the original algorithm of \citet{Hoyer2004} obtained using alternative algorithms. The linear time algorithm proposed in this paper is far superior in terms of execution time. The time complexity of the competing methods is at least quasi-linear, which becomes noticeable especially for large problem dimensionalities.}
  \label{fig:projfunc-timings}
\end{figure}

Figure~\ref{fig:projfunc-timings} visualizes the results of this experiment.
For all tested problem dimensionalities, the here proposed linear time algorithm dominated all previously described methods.
While the speed-up of the algorithms of \citet{Potluru2013} and \citet{Thom2013} relative to the original algorithm of \citet{Hoyer2004} were already significant, especially for small to medium dimensionalities, they got relatively slow for very large vectors.
This is not surprising, because both methods start with sorting the input vector and have to store that permutation to be undone in the end.

The algorithm proposed in this paper is based on root-finding of a monotone function and requires no sorting.
Only the left and right neighbors of a scalar in a vector have to be found.
This can be achieved by scanning linearly through the input vector, which is particularly efficient when huge vectors have to be processed.
For $n = 2^{30}$, the proposed algorithm was more than $15$~times faster than the methods of \citet{Hoyer2004}, \citet{Potluru2013} and \citet{Thom2013}.
Because of this appealing asymptotic behavior, there is now no further obstacle to applying smooth sparseness methods to large-scale problems.

\subsection{Data Set Analysis with EZDL}
\label{sect:data_analysis}
We used the Easy Dictionary Learning algorithm as a data analysis tool to visualize the structure of patches from natural images under different aspects, which facilitates qualitative statements on the algorithm performance.
For our experiments, we used the McGill calibrated color image database \citep{Olmos2004}, where the images had either $768\times 576$ pixels or $576\times 768$ pixels.
The images were desaturated \citepalias[Annex~B.4]{RP177-1993} and quantized to $8$-bit precision.
We extracted $10\;000$~patches with $16\times 16$~pixels from each of the $314$~images of the "Foliage" collection.
The patches were extracted at random positions.
Patches with vanishing variance were omitted since they carry no information.
In total, we obtained $3.1$~million samples for learning.

We learned dictionaries on the raw pixel values of this learning set and on whitened image patches.
The learning samples were normalized to zero mean and unit variance as the only pre-processing for the raw pixel experiments.
The whitened data was obtained using the canonical pre-processing from Section~5.4 of \citet{Hyvaerinen2009} with $128$~principal components.
EZDL was applied using each proposed combination of inference model and atom constraints.
We presented $30\;000$ randomly selected learning samples in each of one thousand epochs, which corresponds to about ten sweeps through the entire learning set.
The initial step size was set to $\eta_0 := 1$.

A noisy version of the final dictionary could already be observed after the first epoch, demonstrating the effectiveness of EZDL for quick analysis.
Since the optimization procedure is probabilistic in the initialization and selection of training samples, we repeated the training five times for each parameter set.
We observed only minor variations between the five dictionaries for each parameter set.

In the following, we present details of dictionaries optimized on raw pixel values and analyze their filter characteristics.
Then, we consider dictionaries learned with whitened data and dictionaries with separable atoms trained on raw pixel values and discuss the underlying reasons for their special morphology.

\subsubsection{Raw Pixel Values}
\label{sect:gabor_ezdl_experiments}
Using the ordinary inference model, we trained two-times overcomplete dictionaries with $n := 256$ atoms on the normalized pixel values, where the dictionary sparseness degree was varied between $0.700$ and $0.999$.
This resulted in the familiar appearance of dictionaries with Gabor-like filters, which resemble the optimal stimulus of visual neurons \citep{Olshausen1996,Olshausen1997}.
While for small values of $\sigma_H$ the filters were mostly small and sharply bounded, high sparseness resulted in holistic and blurry filters.

We used the methods developed by \citet{Jones1987} and \citet{Ringach2002} for a more detailed analysis.
In doing so, a two-dimensional Gabor function as defined in Equation~(1) of \citet{Ringach2002}, that is a Gaussian envelope multiplied with a cosine carrier wave, was fit to each dictionary atom using the algorithm of \citet{Nelder1965}.
We verified that these Gabor fits accurately described the atoms, which confirmed the Gabor-like nature of the filters.

The differences in the dictionaries due to varying sparseness degrees became apparent through analysis of the parameters of the fitted Gabor functions.
Figure~\ref{fig:gabor-params} shows the mean values of three influential Gabor parameters in dependence on $\sigma_H$.
All parameters change continuously and monotonically with increasing sparseness.
The spatial frequency, a factor in the cosine wave of the Gabor function, constantly decreases for increasing dictionary sparseness.
The width of the Gaussian envelope, here broken down into the standard deviation in both principal orientations $x'$ and $y'$, is monotonically increasing.
The envelope width in $y'$ direction increases earlier, but in the end the envelope width in $x'$ direction is larger.

\begin{figure}[t]
  \centering
  \includegraphics[width=84mm]{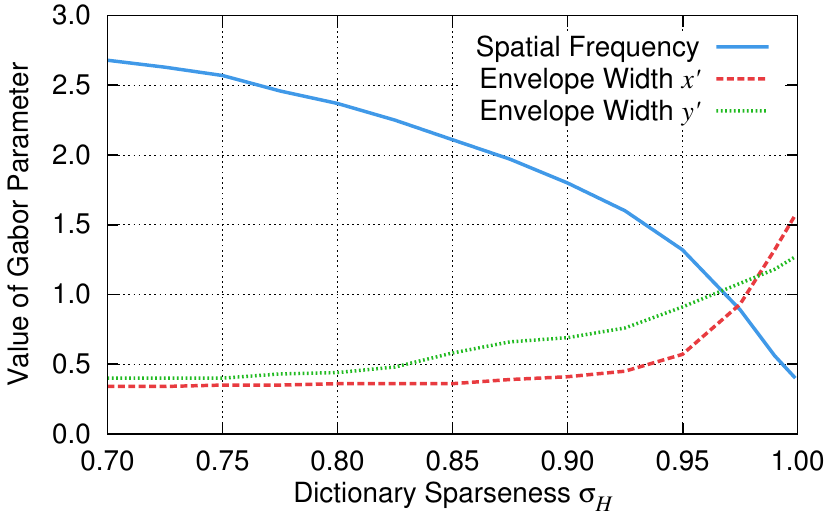}
  \caption{Mean values of three influential parameters of Gabor functions fitted to the atoms of the dictionaries learned with EZDL, in dependence on the target degree of dictionary sparseness $\sigma_H$. An increase in sparseness results in filters with a reduced spatial frequency, but with a significant increase in the width of the Gaussian envelope.}
  \label{fig:gabor-params}
\end{figure}

It can be concluded that more sparse code words result in filters with lower frequency but larger envelope.
Since an increased sparseness reduces the model's effective number of degrees of freedom, it prevents the constrained dictionaries from adapting precisely to the learning data.
Similar to Principal Component Analysis, low-frequency atoms here minimize the reproduction error best when only very few effective atoms are allowed \citep{Hyvaerinen2009,Olshausen1996}.

Histograms of the spatial phase of the filters, an additive term in the cosine wave of the Gabor function, are depicted in Fig.~\ref{fig:phase-histogram}.
For $\sigma_H = 0.75$, there is a peak at $\nicefrac{\pi}{2}$ which corresponds to odd-symmetric filters \citep{Ringach2002}.
The distribution is clearly bimodal for $\sigma_H = 0.99$, with peaks at $0$ and $\nicefrac{\pi}{2}$ corresponding to even-symmetric and odd-symmetric filters, respectively.
While the case of small $\sigma_H$ matches the result of ordinary sparse coding, the higher dictionary sparseness results in filters with the same characteristics as the optimal stimulus of macaque visual neurons \citep{Ringach2002}.

\begin{figure}[t]
  \centering
  \includegraphics[width=84mm]{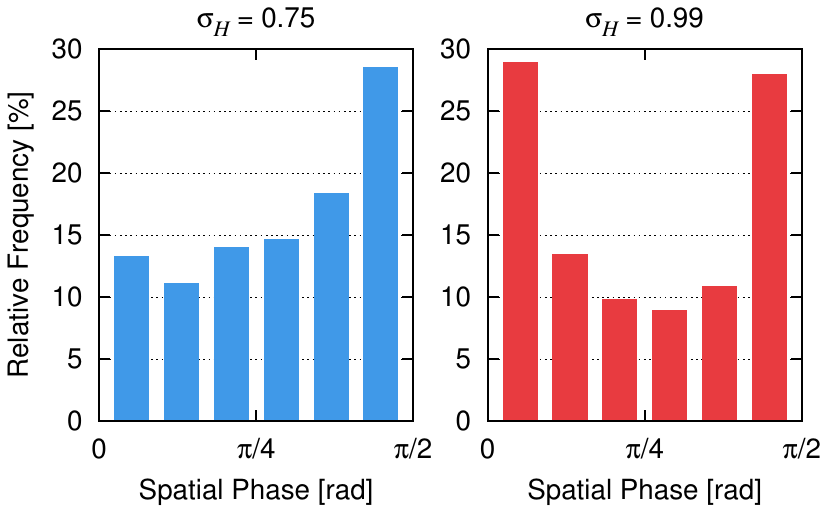}
  \caption{Histograms of the spatial phase of fitted Gabor functions for $\sigma_H = 0.75$ (left) and $\sigma_H = 0.99$ (right). Low sparseness yields dictionaries where odd-symmetric filters dominate, whereas high dictionary sparseness results in a significant amount of both even-symmetric and odd-symmetric filters.}
  \label{fig:phase-histogram}
\end{figure}

This analysis proves that EZDL's minimalistic learning rule is capable of generating biologically plausible dictionaries, which constitute a particularly efficient image coding scheme.
To obtain dictionaries with diverse characteristics, it is enough to adjust the target degree of dictionary sparseness on a normalized scale.
A major component in the learning algorithm is the sparseness projection, enforcing local competition among the atoms \citep{Rozell2008} due to its absolute order-preservation property \citep[Lemma~12]{Thom2013}.

\subsubsection{Whitened Image Patches}
Whitening as a pre-processing step helps to reduce sampling artifacts and decorrelates the input data \citep{Hyvaerinen2009}.
It also changes the intuition of the similarity measure in EZDL's objective function, linear features rather than single pixels are considered where each feature captures a multitude of pixels in the raw patches.
This results in differences in the filter structure, particularly in the emergence of more low-frequency filters.

The dictionary depicted in Fig.~\ref{fig:dictionaries_whitened_topographic} was learned using the topographic inference model with average-pooling of $3\times 3$ neighborhoods.
We set the dictionary sparseness degree to $\sigma_H := 0.85$ and the number of atoms to $n := 256$, arranged on a $16\times 16$ grid.
The dictionary closely resembles those gained through Topographic Independent Component Analysis \citep{Hyvarinen2001,Hyvaerinen2009} and Invariant Predictive Sparse Decomposition \citep{Kavukcuoglu2009}.
It should be noted that here the representation is two times overcomplete due to the whitening procedure.
Overcomplete representations are not inherently possible with plain Independent Component Analysis \citep{Bell1997,Hyvaerinen2009} which limits its expressiveness, a restriction that does not hold for EZDL.

The emergence of topographic organization can be explained by the special design of the topographic inference model.
The pooling operator acts as spatial low-pass filter on the filter responses, so that each smoothed filter response carries information on the neighboring filter responses.
Filter response locality is retained after the sparseness projection, hence adjacent atoms receive similar updates with the Hebbian-like learning rule.
Hence, there are only small differences between filters within the same vicinity.
The learning process is similar to that of Self-organizing Maps \citep{Kohonen1990}, where only the atom with the maximum filter response and the atoms in its direct surrounding are updated.
However, EZDL simultaneously updates multiple clusters since the sparse code words laid out on the grid are multimodal.

Further, it is notable that we achieved a topographic organization merely through a linear operator $G$ by simple average-pooling.
This stands in contrast to the discussion from \citet{Bauer2013} where the necessity of a nonlinearity such as multiplicative interaction or pooling with respect to the Euclidean norm was assumed.
Our result that a simple linear operator plugged into the ordinary inference model already produces smooth topographies proves that linear interactions between the atoms are sufficient despite their minimality.

\newcommand{\dictimgwidth}{67mm}

\begin{figure}[t]
  \centering
  \subfloat[Topographic inference model.]
  {
    \label{fig:dictionaries_whitened_topographic}
    \includegraphics[width=\dictimgwidth]{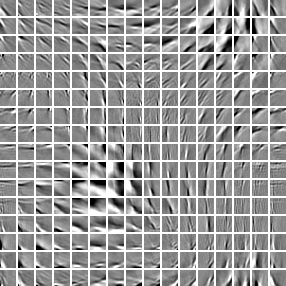}
  }\\
  \subfloat[Rank-1 topographic inference model.]
  {
    \label{fig:dictionaries_whitened_rank_topographic}
    \includegraphics[width=\dictimgwidth]{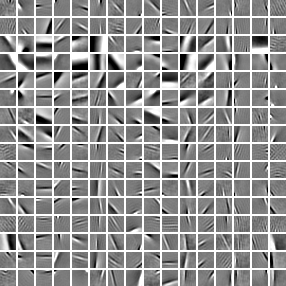}
  }
  \caption{Two times overcomplete dictionaries trained with EZDL on whitened natural image patches.}
  \label{fig:dictionaries_whitened}
\end{figure}

Figure~\ref{fig:dictionaries_whitened_rank_topographic} shows a dictionary obtained with the rank-1 topographic inference model, using a sparseness degree of $\sigma_H := 0.85$ and $n := 256$ filters on a $16\times 16$ grid.
Here, the sparse code words reshaped to a matrix are required to have unit rank which results in a specially organized filter layout.
For example, rows three and four almost exclusively contain low-frequency filters, and the filters in the sixth column are all oriented vertically.
The grouping of similar atoms into the same rows and columns is related to Independent Subspace Analysis, which yields groups of Gabor-like filters where all filters in a group have approximately the same frequency and orientation \citep{Hyvaerinen2000a}.

The rank-1 topographic inference model guarantees that code words can be expressed as the dyadic product of two vectors, where the factors themselves are sparsely populated because the code words are sparse.
This causes the code words to possess a sparse rectangular structure when reshaped to account for the grid layout, that is non-vanishing activity is always concentrated in a few rows and columns.
The Hebbian-like learning rule therefore induces similar updates to atoms located in common rows or columns, which explains the obtained group layout.

\begin{figure}[t]
  \centering
  \subfloat[Ordinary inference model.]
  {
    \label{fig:dictionaries_rank1filters_ordinary}
    \includegraphics[width=\dictimgwidth]{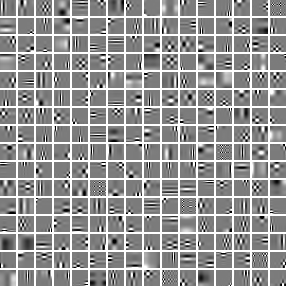}
  }\\
  \subfloat[Topographic inference model.]
  {
    \label{fig:dictionaries_rank1filters_topographic}
    \includegraphics[width=\dictimgwidth]{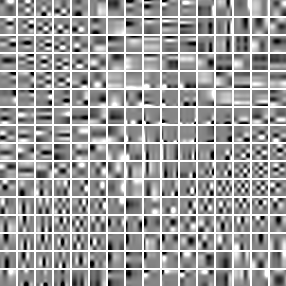}
  }
  \caption{Dictionaries trained with EZDL on raw pixels, the atoms were replaced by their best rank-1 approximation after each epoch.}
  \label{fig:dictionaries_rank1filters}
\end{figure}

\subsubsection{Rank-1 Filters on Raw Pixel Values}
\label{sect:rank1_ezdl_experiments}
Enforcing the filters themselves to have rank one by setting $\kappa_W := 1$ resulted in bases similar to Discrete Cosine Transform \citep{Watson1994}.
This was also observed recently by \citet{Hawe2013} who considered a tensor decomposition of the dictionary, and by \citet{Rigamonti2013} who minimized the Schatten $1$-norm of the atoms.
Note that EZDL merely replaces the atoms after each learning epoch by their best rank-1 approximation.
The computational complexity of this operation is negligible compared to an individual learning epoch using tens of thousands of samples and hence does not slow down the actual optimization.

If no ordering between the filters was demanded, a multitude of checkerboard-like filters and vertical and horizontal contrast bars was obtained, see Fig.~\ref{fig:dictionaries_rank1filters_ordinary} for a dictionary with $n := 256$ atoms and a sparseness degree of $\sigma_H := 0.85$.
Atoms with diagonal structure cannot be present at all in such constrained dictionaries because diagonality requires a filter rank greater than one.
Although all filters were paraxial due to the rank-1 constraint, they still resembled contrast fields because of their grating-like appearance.
This is due to the representation's sparseness, which is known to induce this appearance.

A similar filter morphology was obtained with a topographic filter ordering using a $16\times 16$ grid, though the filters were blurrier, as shown in Fig.~\ref{fig:dictionaries_rank1filters_topographic}.
Here, checkerboard-like filters are located in clusters, and filters with vertical and horizontal orientation are spatially adjacent.
This is as expected, because with a sparse topography there are only a few local blobs of active entries in each code word, causing similar atoms to be grouped together and dissimilar atoms to be distant from each other.
In the space of rank-1 filters, there can be either vertical or horizontal structures, or checkerboards combining both principal orientations.

The lack of fine details can be explained by the restriction of topographic organization, which reduces the effective number of degrees of freedom.
Analogously to the situation in Sect.~\ref{sect:gabor_ezdl_experiments} where the dictionary sparseness was increased, the reproduction error can be reduced by a large amount using few filters with low frequencies.
For the same reason, there are few such checkerboard-like filters:
Minimization of the reproduction error is first achieved with principal orientations before checkerboards can be used to encode details in the image patches.

In conclusion, these results show that the variety of the dictionaries produced by EZDL can be vast simply by adapting easily interpretable parameters.
The algorithm covers and reproduces well-known phenomena from the literature, and allows a precise visualization of a data set's structure.
A first impression of the final dictionaries can already be obtained after one learning epoch, which takes no more than one second on a modern computer.

\subsection{Experiments on the Reproduction of Entire Images}
\label{sect:reproduction_experiments}
\begin{figure*}[t]
  \centering
  \subfloat[Reproduction quality in dependence on the number of Landweber iterations carried out for inference. The four curves show different combinations of dictionary sparseness $\sigma_H$ and inference sparseness $\sigma_I$.]
  {
    \label{fig:landweber_iteration_plot_256}
    \includegraphics[width=84mm]{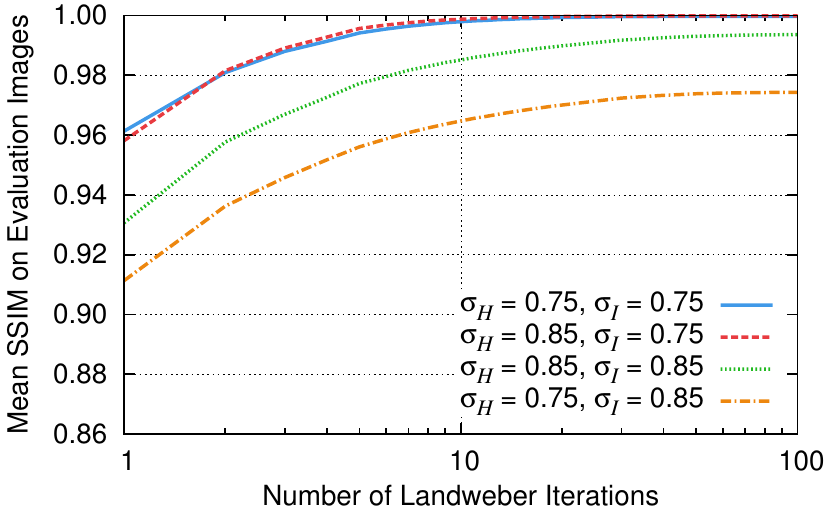}
  }\hfill
  \subfloat[Achieved reproduction quality in dependence on the inference sparseness $\sigma_I$ for four different dictionary sparseness degrees $\sigma_H$. Large values of $\sigma_H$ yield the best performance for large $\sigma_I$.]
  {
    \label{fig:inference_sparseness_plot_256}
    \includegraphics[width=84mm]{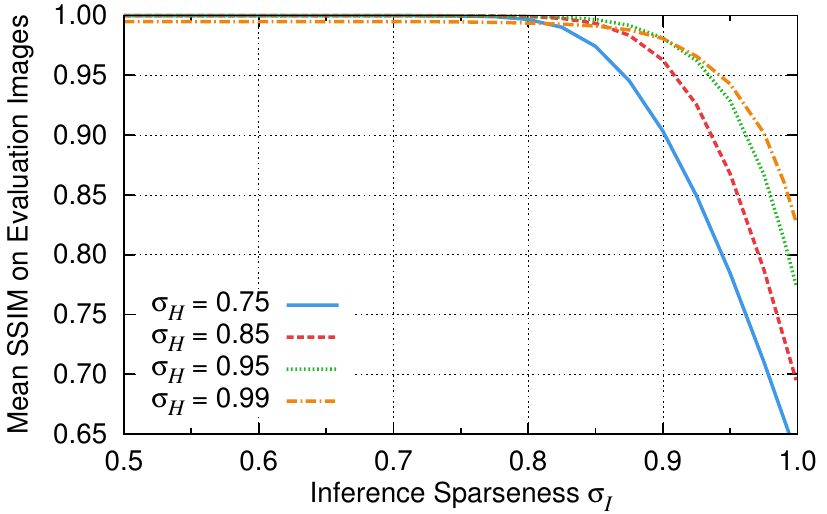}
  }
  \caption{Results on the reproduction of entire images using Easy Dictionary Learning dictionaries. \protect\subref{fig:landweber_iteration_plot_256}~While one Landweber iteration suffices for learning, better reproduction performance is yielded if the Landweber procedure is run until convergence. \protect\subref{fig:inference_sparseness_plot_256}~The relation between the dictionary sparseness $\sigma_H$ and the inference sparseness $\sigma_I$ is influential. Hardly any information loss can be observed for $\sigma_I < \sigma_H$, while the case $\sigma_I > \sigma_H$ shows dictionaries should be trained for very sparse code words when demanding such a sparseness for reproduction.}
  \label{fig:inference_plots}
\end{figure*}

It was demonstrated in Sect.~\ref{sect:data_analysis} that the Easy Dictionary Learning algorithm produces dictionaries that correspond to efficient image coding schemes.
We now analyze the suitability of EZDL dictionaries trained on raw pixel values for the reproduction of entire images.
This is a prerequisite for several image processing tasks such as image enhancement and compression, since here essentially the same optimization problem should be solved as during reproduction with a given dictionary.

Our analysis allows quantitative statements as the original images and their reproduction through sparse coding can be numerically compared on the signal level.
Further, the EZDL dictionaries are compared with the results of the Online Dictionary Learning algorithm by \citet{Mairal2009a} and the Recursive Least Squares Dictionary Learning Algorithm by \citet{Skretting2010} in terms of reproduction quality and learning speed.

The evaluation methodology was as follows.
First, dictionaries with different parameter sets were trained on $3.1$ million $8\times 8$ pixels natural image patches extracted from the "Foliage" collection of the McGill calibrated color image database \citep{Olmos2004}.
As in Sect.~\ref{sect:data_analysis}, $10\;000$ patches were picked from random positions from each of the desaturated and quantized images.
All patches were normalized to zero mean and unit variance before training.
The dictionaries were designed for a four times overcomplete representation, they had $n := 256$ atoms for samples with $64$~pixels.
After training, dictionary quality was measured using the $98$~images of the "Flowers" collection from the same image database.
This evaluation set was disjoint from the learning set.

Each single image was subsequently divided into non-overlapping blocks with $8\times 8$ pixels.
All blocks were normalized to zero mean and unit variance, and then the optimized dictionaries were used to infer a sparse code word for each block.
The resulting code word was then fed through a linear generative model using the currently investigated dictionary, and the mean value and variance were restored to match that of the original block.

Sparse code word inference was achieved with a projected Landweber procedure \citep{Bredies2008}, essentially the same as projected gradient descent starting with the null vector.
Using the sparseness-enforcing projection operator after each iteration, the code words were tuned to attain an \emph{inference sparseness} $\sigma_I$, which need not necessarily be equal to the sparseness degree used for dictionary learning.
The correlation coefficient was used as the similarity measure for inference to benefit from invariance to shifting and scaling.
Automatic step size control was carried out with the bold driver heuristic \citep{Bishop1995}.
We note that due to the representation theorem on the sparseness projection this process can also be understood as Iterative Soft-Thresholding \citep{Bredies2008}.

A reproduced image is yielded by applying this procedure to all distinct blocks of the original image using a single dictionary.
The deviation between this output image and the original image was assessed with the SSIM index \citep{Wang2009}, which yielded qualitatively similar results as the peak signal-to-noise ratio.
The SSIM index is, however, normalized to values between zero and one, and it respects the local structure of images since it examines $11\times 11$ pixels neighborhoods through convolution.
Because it measures the visual similarity between images, it is more intuitive than measures based on the pixel-wise squared error \citep{Wang2009}.
This evaluation method yielded one number from the interval $\intervalcc{0}{1}$ for each image and dictionary.
Each parameterization for dictionary learning was used to train five dictionaries to compensate probabilistic effects, and here the mean of the resulting $490$~SSIMs is reported.

\subsubsection{EZDL Results}
Figure~\ref{fig:inference_plots} visualizes the results obtained for dictionaries produced by Easy Dictionary Learning using dictionary sparseness degrees $\sigma_H\in\set{0.75,\ 0.85,\ 0.95,\, 0.99}$.
One thousand learning epochs were carried out, presenting $30\;000$ samples in each epoch, and using an initial step size of $\eta_0 := 1$.
During dictionary learning, only the first trivial iteration of a Landweber procedure is carried out for sparse code word inference by computing $h := \Pi_{\sigma_H}(W\transp x)$ for the ordinary inference model.

\begin{figure*}[t]
  \centering
  \subfloat[Results of the Online Dictionary Learning (ODL) algorithm by \citet{Mairal2009a}, where $\lambda$ trades off between sparseness and reproduction capabilities in a model with implicit sparseness constraints.]
  {
    \label{fig:inference_sparseness_plot_256_spams}
    \includegraphics[width=84mm]{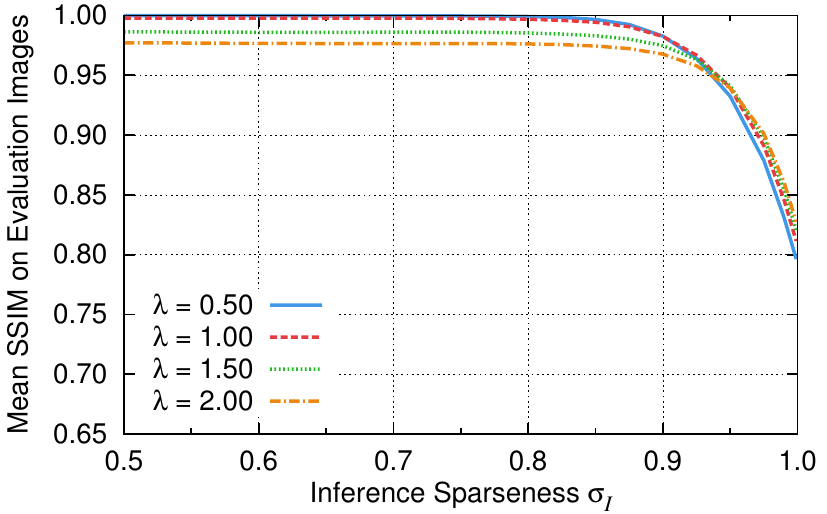}
  }\hfill
  \subfloat[Results of the Recursive Least Squares Dictionary Learning Algorithm (RLS-DLA) by \citet{Skretting2010}. Here, $\zeta$ denotes the target $L_0$ pseudo-norm of the code words for inference.]
  {
    \label{fig:inference_sparseness_plot_256_karlsk}
    \includegraphics[width=84mm]{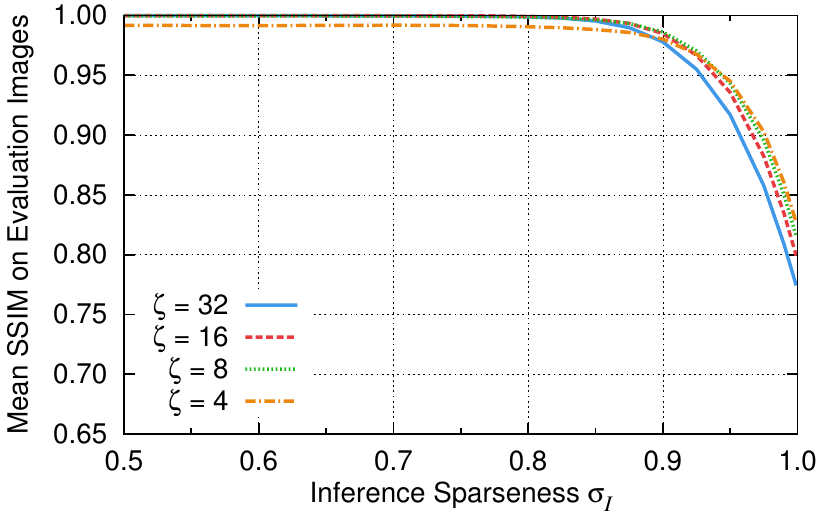}
  }
  \caption{Experimental results on image reproduction where alternative dictionary learning algorithms were used. The dictionaries behave similarly to Easy Dictionary Learning dictionaries if the dictionary sparseness parameter and the inference sparseness are varied.}
  \label{fig:inference_plots_comparison}
\end{figure*}

We first analyzed the impact of carrying out more Landweber iterations during the image reproduction phase, and the effect of varying the inference sparseness degree $\sigma_I$.
A huge performance increase can be obtained by using more than one iteration for inference (Fig.~\ref{fig:landweber_iteration_plot_256}).
Almost the optimal performance is achieved after ten iterations, and after one hundred iterations the method has converged.
For $\sigma_I = 0.75$, the performance of dictionaries with $\sigma_H = 0.75$ and $\sigma_H = 0.85$ is about equal yielding the maximum SSIM of one.
This value indicates that there is no visual difference between the reproduced images and the original images.

When the inference sparseness $\sigma_I$ is increased to $0.85$, a difference in the choice of the dictionary sparseness $\sigma_H$ becomes noticeable.
For $\sigma_H = 0.85$, performance already degrades, and the degradation is substantial if $\sigma_H = 0.75$.
It is more intuitive when $\sigma_H$ and $\sigma_I$ are put in relation.
Almost no information is lost in the reproduction by enforcing a lower inference sparseness than dictionary sparseness ($\sigma_I < \sigma_H$).
Performance is worse if $\sigma_I > \sigma_H$ which is plausible because the dictionary was not adapted for a higher sparseness degree.
In the natural case $\sigma_I = \sigma_H$ the performance mainly depends on their concrete value, where higher sparseness results in worse reproduction capabilities.

To further investigate this behavior, we set the number of Landweber iterations to $100$ and varied $\sigma_I$ smoothly in the interval $\intervalco{0.50}{1.0}$ for the four different dictionary sparseness degrees.
The results of this experiment are depicted in Fig.~\ref{fig:inference_sparseness_plot_256}.
For $\sigma_I \leq 0.80$ there is hardly any difference in the reproduction quality irrespective of the value of $\sigma_H$.
The difference when training dictionaries with different sparseness degrees first becomes visible for large values of $\sigma_I$.
Performance is here better using dictionaries where very sparse code words were demanded during learning.
Hence, tasks that require high code word sparseness should use dictionaries trained with high values of the dictionary sparseness.

\subsubsection{Comparison with Alternative Dictionary Learning Algorithms}
For a comparison, we conducted experiments using the Online Dictionary Learning (ODL) algorithm of \citet{Mairal2009a} and the Recursive Least Squares Dictionary Learning Algorithm (RLS-DLA) by \citet{Skretting2010}.
For inference of sparse code words, ODL minimizes the reproduction error under implicit sparseness constraints.
RLS-DLA uses an external vector selection algorithm for inference, hence explicit constraints such as a target $L_0$ pseudo-norm can be demanded easily.
Both algorithms update the dictionary after each sample presentation.

ODL does not require any step sizes to be adjusted.
The crucial parameter for dictionary sparseness here is a number $\lambda\in\R_{>0}$, which controls the trade-off between reproduction capabilities and code word sparseness.
We trained five dictionaries for each $\lambda\in\set{0.50,\ 1.00,\ 1.50,\ 2.00}$ using $n := 256$ atoms and presenting $30\;000$ learning samples in each of one thousand epochs.
Then, the same evaluation methodology as before with $\sigma_I\in\intervalco{0.50}{1.0}$ was used to assess the reproduction capabilities.
The results are shown in Fig.~\ref{fig:inference_sparseness_plot_256_spams}.
The choice of $\lambda$ is not as influential compared to $\sigma_H$ in EZDL.
There are only small performance differences for $\sigma_I < 0.925$, and for $\sigma_I \geq 0.925$ hardly any difference can be observed.
Similar to the EZDL experiments, large values of $\lambda$ result in better performance for large $\sigma_I$ and worse performance for small $\sigma_I$.

RLS-DLA provides a forgetting factor parameter which behaves analogously to the step size in gradient descent.
We used the default forgetting factor schedule which interpolates between $0.99$ and $1$ using a cubic function over one thousand learning epochs.
For inference, we chose an optimized Orthogonal Matching Pursuit variant \citep{Gharavi1998} where a parameter $\zeta\in\N$ controls the number of non-vanishing coordinates in the code words.
We trained five dictionaries with $n := 256$ atoms for each $\zeta\in\set{4,\ 8,\ 16,\ 32}$.
Thirty thousand randomly drawn learning samples were presented in each learning epoch.
The resulting reproduction performance is shown in Fig.~\ref{fig:inference_sparseness_plot_256_karlsk}.
Again, for large values of $\sigma_I$ the dictionaries with the most sparse code words during learning perform best, that is those trained with small values of $\zeta$.

To compare all three dictionary learning algorithms independent of the concrete dictionary sparseness, we took the mean SSIM value that belonged to the best performing dictionaries for each feasible value of $\sigma_I$.
This can also be interpreted as using the convex hull of the results shown in Fig.~\ref{fig:inference_sparseness_plot_256} and Fig.~\ref{fig:inference_plots_comparison}.
This yielded one curve for each dictionary learning algorithm, depicted in Fig.~\ref{fig:inference_plots_maxoverdicts}.
There is only a minor difference between the curves over the entire range of $\sigma_I$.
It can hence be concluded that the algorithms learn dictionaries which are equally well-suited for the reproduction of entire images.
Although EZDL uses a very simple learning rule, this is sufficient enough to achieve a performance competitive with that of two state-of-the-art dictionary learning algorithms.

\subsubsection{Comparison of Learning Speed}
We moreover compared the learning speed of the methods and investigated the influence of the initial step size on the EZDL dictionaries.
In doing so, we carried out reproduction experiments on entire images using dictionaries provided by the learning algorithms after certain numbers of learning epochs.
In each of one thousand overall epochs, $30\;000$~learning samples where input to all three algorithms.
The dictionary sparseness parameters were set to $\sigma_H := 0.99$ for EZDL, to $\lambda := 2.00$ for Online Dictionary Learning, and to $\zeta := 4$ for the Recursive Least Squares Dictionary Learning Algorithm.

\begin{figure}[t]
  \centering
  \includegraphics[width=84mm]{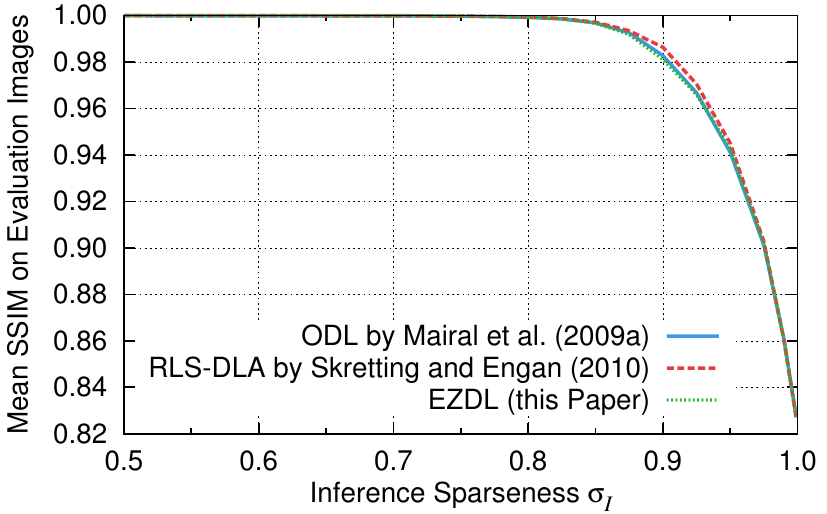}
  \caption{Comparison of dictionary performance when the dependency on the dictionary sparseness degrees $\lambda$, $\zeta$ and $\sigma_H$ was eliminated, for ODL, RLS-DLA and EZDL, respectively. All three algorithms produce dictionaries equally well-suited to the reproduction of entire images.}
  \label{fig:inference_plots_maxoverdicts}
\end{figure}

\begin{figure}[t]
  \centering
  \includegraphics[width=84mm]{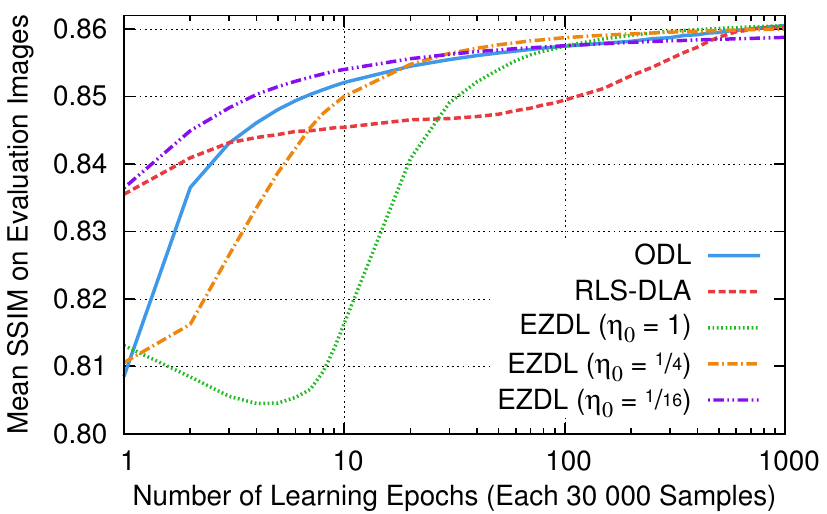}
  \caption{Reproduction quality in dependence on the number of samples presented to Online Dictionary Learning, the Recursive Least Squares Dictionary Learning Algorithm, and Easy Dictionary Learning, where the initial step size $\eta_0$ was varied for the latter. All approaches achieved the same performance after $1000$~learning epochs.}
  \label{fig:learning-history}
\end{figure}

Timing measurements on an Intel Core i7-4960X processor have shown that one EZDL learning epoch takes approximately $30\;\%$ less time than a learning epoch of ODL.
RLS-DLA was roughly $25$~times slower than EZDL.
Although the employed vector selection method was present as part of an optimized software library, only a fairly unoptimized implementation of the actual learning algorithm was available.
Therefore, a comparison of this algorithm with regard to execution speed would be inequitable.

The inference sparseness was set to $\sigma_I := 0.99$ and one hundred Landweber iterations were carried out for sparse code word inference.
The SSIM index was eventually used to assess the reproduction performance.
Figure~\ref{fig:learning-history} visualizes the results of the learning speed analysis, averaged over all images from the evaluation set and five dictionaries trained for each parameterization to compensate probabilistic effects.
Online Dictionary Learning is free of step sizes, dictionary performance increased instantly and constantly with each learning epoch.
Only small performance gains could be observed after $100$~learning epochs.

The performance of dictionaries trained with RLS-DLA was initially better than that of the ODL dictionaries, but improved more slowly.
After a hundred epochs, however, it was possible to observe a significant performance gain.
Although tweaking the forgetting factor schedule may have led to better early reproduction capabilities, RLS-DLA achieved a performance equivalent to that of ODL after $1000$~epochs.

If EZDL's initial step size was set to unity, performance first degraded, but started to improve after the fifth epoch.
After $100$~epochs, the performance was identical to ODL's, and all three methods obtained equal reproduction capabilities after $1000$~epochs.
Reduction of the initial step size caused the dictionaries to perform better after few sample presentations.
The quality of Online Dictionary Learning was achieved after $20$~epochs for $\eta_0 = \nicefrac{1}{4}$, and for $\eta_0 = \nicefrac{1}{16}$ the EZDL dictionaries were always better in the mean until the $100$th epoch.
There is hardly any quality difference after $1000$~epochs, a very small initial step size resulted however in a slightly worse performance.

The choice of the initial step size for EZDL is hence not very influential.
Although $\eta_0 = 1$ caused small overfitting during the very first epochs, the dictionaries quickly recovered so that no significant difference in reproduction capabilities could be observed after $100$~epochs.
Since an EZDL epoch is $30\;\%$ faster than an Online Dictionary Learning epoch, our proposed method produces better results earlier on an absolute time scale.

\subsection{Image Denoising Experiments}
We have demonstrated in Sect.~\ref{sect:reproduction_experiments} that dictionaries learned with Easy Dictionary Learning are as good as those obtained from the Online Dictionary Learning algorithm of \citet{Mairal2009a} and the Recursive Least Squares Dictionary Learning Algorithm by \citet{Skretting2010} in terms of the reproduction quality of entire images.
In a final experiment, we investigated the suitability of EZDL dictionaries for image denoising using the image enhancement procedure proposed by \citet{Mairal2009}.

This method carries out semi-local block matching and finds sparse code words by imposing a group sparseness penalty on the Euclidean reproduction error.
In doing so, a pre-learned dictionary is used which explains the appearance of uncorrupted images and further helps to resolve ambiguities if block matching fails to provide large enough groups.
A linear generative model is finally used to find estimators of denoised image patches from the sparse code words.
This procedure is quite robust if the input data is noisy, since sparseness provides a strong prior which well regularizes this ill-posed inverse problem \citep{Kreutz-Delgado2003,Foucart2013}.

The denoising approach of \citet{Mairal2009} also provides the possibility of dictionary adaptation while denoising concrete input data.
We did not use this option as it would hinder resilient statements on a dictionary's eligibility if it would be modified with another dictionary learning algorithm during denoising.

The methodology of the experiment was as follows.
We used the four-times overcomplete dictionaries trained on the $8\times 8$ pixels image patches from the "Foliage" collection of the McGill calibrated color image database \citep{Olmos2004} as models for uncorrupted images.
The dictionary sparseness parameters were $\sigma_H := 0.99$ for EZDL, $\lambda := 2.00$ for ODL and $\zeta := 4$ for RLS-DLA.
For evaluation, we used the $81$~images of the "Animals" collection from the McGill database, and converted them to $8$-bit grayscales as in the previous experiments.
The images were synthetically corrupted with additive white Gaussian noise.
Five noisy images were generated for each original image and each standard deviation from $\set{2.5,\ 5.0,\ \dotsc,\ 47.5,\ 50.0}$.

These images were then denoised using the candidate dictionaries, where the window size for block matching was set to $32$~pixels.
Then, the similarity between the reconstructions and the original images was measured with the peak signal-to-noise ratio.
We also evaluated the SSIM index, which led to qualitatively similar results.

\begin{figure}[t]
  \centering
  \includegraphics[width=84mm]{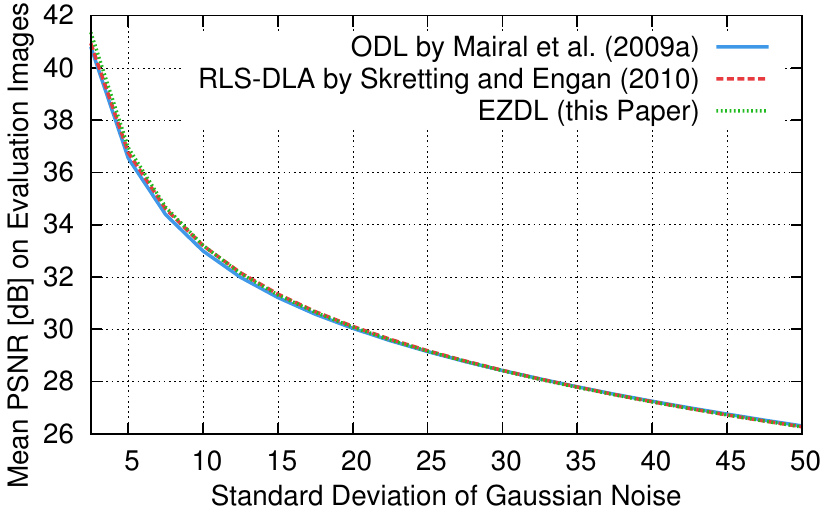}
  \caption{Denoising performance in terms of the peak signal-to-noise ratio (PSNR) using the denoising method proposed by \citet{Mairal2009}. There is only a small performance difference between dictionaries trained with ODL, RLS-DLA and EZDL.}
  \label{fig:denoise-results}
\end{figure}

The results are depicted in Fig.~\ref{fig:denoise-results}.
Denoising performance degrades if the noise strength is increased.
There is hardly any difference between the dictionaries trained with the three algorithms.
The RLS-DLA and EZDL dictionaries perform slightly better for small synthetic noise levels, but this improvement is visually imperceptible.
This result is not surprising, since Sect.~\ref{sect:reproduction_experiments} demonstrated that all three learning algorithms produce dictionaries equally well-suited for the reproduction of entire images.

The denoising procedure of \citet{Mairal2009} aims at reproduction capabilities as well, with the modification of employing noisy samples as input.
Image enhancement and compression applications such as those proposed by \citet{Yang2010,Yang2012}, \citet{Dong2011}, \citet{Skretting2011} and \citet{Horev2012} which also use problem formulations based on the reproduction error can hence be expected to benefit from more efficiently learned dictionaries as well.

\section{Conclusions}
\label{sect:conclusions}
This paper proposed the EZDL algorithm which features explicit sparseness constraints with respect to Hoyer's smooth sparseness measure $\sigma$.
Pre-defined sparseness degrees are ensured to always be attained using a sparseness-enforcing projection operator.
Building upon a succinct representation of the projection, we proved that the projection problem can be formulated as a root-finding problem.
We presented a linear time and constant space algorithm for the projection which is superior to previous approaches in terms of theoretical computational complexity and execution time on real computing machines.

EZDL adapts dictionaries to measurement data through simple rank-1 updates.
The sparseness projection serves as foundation for sparse code word inference.
Due to the projection efficiency and since no complicated gradients are required, our proposed learning algorithm is significantly faster than even the optimized ODL algorithm.
Topographic atom organization and atom sparseness can be realized with very simple extensions, allowing for versatile sparse representations of data sets.
Its simplicity and efficiency does not hinder EZDL from producing dictionaries competitive with those generated by ODL and RLS-DLA in terms of the reproduction and denoising performance on natural images.
Alternative image processing methods based on sparse representations rely on dictionaries subject to the same criteria, and can thus be expected to benefit from EZDL's advantages as well.

\begin{acknowledgements}
The authors are grateful to Heiko Neumann, Florian Sch\"ule, and Michael Gabb for helpful discussions.
We would like to thank Julien Mairal and Karl Skretting for making implementations of their algorithms available.
Parts of this work were performed on the computational resource bwUniCluster funded by the Ministry of Science, Research and Arts and the Universities of the State of Baden-W\"urttemberg, Germany, within the framework program bwHPC.
This work was supported by Daimler~AG, Germany.
\end{acknowledgements}

\appendix
\section*{Appendix: Technical Details and Proofs for Section~\ref{sect:projfunc}}
This appendix studies the algorithmic computation of Euclidean projections onto level sets of Hoyer's $\sigma$ in greater detail, and in particular proves the correctness of the algorithm proposed in Sect.~\ref{sect:projfunc}.

For a non-empty subset $M\subseteq\R^n$ of the Euclidean space and a point $x\in\R^n$, we call
\begin{displaymath}
  \proj_M(x) := \set{y\in M | \norm{y - x}_2 = \inf\nolimits_{z\in M}\norm{z - x}_2}
\end{displaymath}
the set of Euclidean projections of $x$ onto $M$ \citep{Deutsch2001}.
Since we only consider situations in which $\proj_M(x) = \set{y}$ is a singleton, we may also write $y = \proj_M(x)$.

Without loss of generality, we can compute $\proj_T(x)$ for a vector $x\in\R_{\geq 0}^n$ within the non-negative orthant instead of $\proj_S(x)$ for an arbitrary point $x\in\R^n$ to yield sparseness-enforcing projections, where $T$ and $S$ are as defined in Sect.~\ref{sect:projfunc}.
First, the actual scale is irrelevant as we can simply re-scale the result of the projection \citep[Remark~5]{Thom2013}.
Second, the constraint that the projection lies in the non-negative orthant $\R_{\geq 0}^n$ can easily be handled by flipping the signs of certain coordinates \citep[Lemma~11]{Thom2013}.
Finally, all entries of $x$ can be assumed non-negative with Corollary~19 from \citet{Thom2013}.

We note that $T$ is non-convex because of the $\norm{s}_2 = \lambda_2$ constraint.
Moreover, $T\neq\emptyset$ for all target sparseness degrees $\sigma^*\in\intervaloo{0}{1}$ which we show here by construction (see also Remark~18 in \citet{Thom2013} for further details):
Let $\psi := \big(\lambda_1 - \nicefrac{\sqrt{n\lambda_2^2 - \lambda_1^2}}{\sqrt{n - 1}}\big)/n > 0$ and $\omega := \lambda_1 - (n - 1)\psi > 0$, then the point $q := \sum_{i = 1}^{n - 1} \psi e_i + \omega e_n\in\R^n$ lies in $T$, where $e_i\in\R^n$ denotes the $i$-th canonical basis vector.
Since all coordinates of $q$ are positive, $T$ always contains points with an $L_0$ pseudo-norm of $n$.
If we had used the $L_0$ pseudo-norm to measure sparseness, then $q$ would have the same sparseness degree as, for example, the vector with all entries equal to unity.
If, however, $\sigma^*$ is close to one, then there is only one large value $\omega$ in $q$ and all the other entries equaling $\psi$ are very small but positive.
This simple example demonstrates that in situations where the presence of noise cannot be eliminated, Hoyer's $\sigma$ is a much more robust sparseness measure than the $L_0$ pseudo-norm.

\subsection*{Representation Theorem}
Before proving a theorem on the characterization of the projection onto $T$, we first fix some notation.
As above, let $e_i\in\R^n$ denote the $i$-th canonical basis vector and let furthermore $e := \sum_{i=1}^ne_i\in\R^n$ be the vector where all entries are one.
We note that if a point $x$ resides in the non-negative orthant, then $\norm{x}_1 = e\transp x$.
Subscripts to vectors denote the corresponding coordinate, except for $e$ and $e_i$.
For example, we have that $x_i = e_i\transp x$.
We abbreviate $\xi\in\R_{\geq 0}$ with $\xi\geq 0$ when it is clear that $\xi$ is a real number.
When $I\subseteq\discint{1}{n}$ is an index set with $d := \abs{I}$ elements, say $i_1 < \cdots < i_d$, then the unique matrix $V_I\in\set{0, 1}^{d\times n}$ with $V_I x = \left(x_{i_1},\ \dots,\ x_{i_d}\right)\transp\in\R^d$ for all $x\in\R^n$ is called the \emph{slicing operator}.
A useful relation between the $L_0$ pseudo-norm, the Manhattan norm and the Euclidean norm is $\norm{x}_2 \leq \norm{x}_1 \leq \smash{\norm{x}_0^{\nicefrac{1}{2}}\norm{x}_2} \leq \sqrt{n}\norm{x}_2$ for all points $x\in\R^n$.

We are now in a position to formalize the representation theorem:
\begin{theorem}
\label{thm:representation}
Let $x\in\R_{\geq 0}^n\setminus T$ and $p := \proj_T(x)$ be unique.
Then there is exactly one number $\alpha^*\in\R$ such that
\begin{displaymath}
  p = \beta^*\cdot\max\left(x - \alpha^*\cdot e,\ 0\right)\text{,}
\end{displaymath}
where $\beta^* := \nicefrac{\lambda_2}{\norm{\max\left(x - \alpha^*\cdot e,\ 0\right)}_2} > 0$ is a scaling constant.
Moreover, if $I := \set{i\inint{1}{n} | p_i > 0} = \discint{i_1}{i_d}$, $d := \abs{I}$ and $i_1 < \cdots < i_d$, denotes the set of the $d$ coordinates in which $p$ does not vanish, then
\begin{displaymath}
  \alpha^* = \frac{1}{d}\left(\norm{V_I x}_1 - \lambda_1\sqrt{\frac{d\norm{V_I x}_2^2 - \norm{V_I x}_1^2}{d\lambda_2^2 - \lambda_1^2}}\right)\text{.}
\end{displaymath}
\end{theorem}

\begin{proof}
It is possible to prove this claim either constructively or implicitly, where both variants differ in whether the set $I$ of all positive coordinates in the projection can be computed from $x$ or must be assumed to be known.
We first present a constructive proof based on a geometric analysis conducted in \citet{Thom2013}, which contributes to deepening our insight into the involved computations.
As an alternative, we also provide a rigorous proof using the method of Lagrange multipliers which greatly enhances the unproven analysis of \citet[Section~3.1]{Potluru2013}.

We first note that when there are $\alpha^*\in\R$ and $\beta^* > 0 $ so that we have $p = \beta^*\cdot\max\left(x - \alpha^*\cdot e,\ 0\right)$, then $\beta^*$ is determined already through $\alpha^*$ because it holds that $\lambda_2 = \norm{p}_2 = \beta^*\cdot\norm{\max\left(x - \alpha^*\cdot e,\ 0\right)}_2$.
We now show that the claimed representation is unique, and then present the two different proofs for the existence of the representation.

\emph{Uniqueness:}
It is $d\geq 2$, since $d = 0$ would violate that $\norm{p}_1 > 0$ and $d = 1$ is impossible because $\sigma^* \neq 1$.
We first show that there are two distinct indices $i,j\in I$ with $p_i\neq p_j$.
Assume this was not the case, then $p_i =: \gamma$, say, for all $i\in I$.
Let $j := \argmin_{i\in I} x_i$ be the index of the smallest coordinate of $x$ which has its index in $I$.
Let $\psi := \big(\lambda_1 - \nicefrac{\sqrt{d\lambda_2^2 - \lambda_1^2}}{\sqrt{d - 1}}\big)/d\in\R$ and $\omega := \lambda_1 - (d - 1)\psi\in\R$ be numbers and define $s := \sum_{i\in I\setminus\set{j}}\psi e_i + \omega e_j\in\R^n$.
Then $s\in T$ like in the argument where we have shown that $T$ is non-empty.
Because of $\norm{p}_1 = \norm{s}_1$ and $\norm{p}_2 = \norm{s}_2$, it follows that
$\norm{x - p}_2^2 - \norm{x - s}_2^2 = 2x\transp\left(s - p\right) = 2\sum_{i\in I\setminus\set{j}} x_i(\psi - \gamma) + 2x_j (\omega - \gamma) \geq 2x_j((d-1)\psi + \omega - d\gamma) = 2x_j\left(\norm{s}_1 - \norm{p}_1\right) = 0$.
Hence $s$ is at least as good an approximation to $x$ as $p$, violating the uniqueness of $p$.
Therefore, it is impossible that the set $\set{p_i | i\in I}$ is a singleton.

Now let $i,j\in I$ with $p_i\neq p_j$ and $\alpha_1^*,\alpha_2^*,\beta_1^*,\beta_2^*\in\R$ such that $p = \beta_1^*\cdot\max\left(x - \alpha_1^*\cdot e,\ 0\right) = \beta_2^*\cdot\max\left(x - \alpha_2^*\cdot e,\ 0\right)$.
Clearly $\beta_1^*\neq 0$ and $\beta_2^*\neq 0$ as $d\neq 0$.
It is $0\neq p_i - p_j = \beta_1^*(x_i - \alpha_1^*) - \beta_1^*(x_j - \alpha_1^*) = \beta_1^*(x_i - x_j)$, thus $x_i\neq x_j$ holds.
Moreover, $0 = p_i - p_j + p_j - p_i = (\beta_1^* - \beta_2^*)(x_i - x_j)$, hence $\beta_1^* = \beta_2^*$.
Finally, we have that $0 = p_i - p_i = \beta_1^*(x_i - \alpha_1^*) - \beta_2^*(x_i - \alpha_2^*) = \beta_1^*(\alpha_2^* - \alpha_1^*)$, which yields $\alpha_1^* = \alpha_2^*$, and hence the representation is unique.

\emph{Existence (constructive):}
Let $H := \set{a\in\R^n | e\transp a = \lambda_1}$ be the hyperplane on which all points in the non-negative orthant have an $L_1$ norm of $\lambda_1$ and let $C := \R_{\geq 0}^n\cap H$ be a scaled canonical simplex.
Further, let $L := \set{q\in H | \norm{q}_2 = \lambda_2}$ be a circle on $H$, and for an arbitrary index set $I\subseteq\discint{1}{n}$ let $L_I := \set{a\in L | a_i = 0\text{ for all }i\not\in I}$ be a subset of $L$ where all coordinates not indexed by $I$ vanish.
With Theorem~2 and Appendix~D from \citet{Thom2013} there exists a finite sequence of index sets $I_1,\dotsc,I_h\subseteq\discint{1}{n}$ with $I_j\supsetneq I_{j+1}$ for $j\inint{1}{h-1}$ such that $\proj_T(x)$ is the result of the finite sequence
\begin{align*}
  r(0) &:= \proj_H(x)\text{, }& s(0) &:= \proj_L(r(0))\text{,}\\
  r(1) &:= \proj_C(s(0))\text{, }& s(1) &:= \proj_{L_{I_1}}\!(r(1))\text{,}\enspace\dots\\
  r(j) &:= \proj_C(s(j-1))\text{, }& s(j) &:= \proj_{L_{I_j}}\!(r(j))\text{,}\enspace\dots\\
  r(h) &:= \proj_C(s(h-1))\text{, }& s(h) &:= \proj_{L_{I_h}}\!(r(h)) = p\text{.}
\end{align*}
All intermediate projections yield unique results because $p$ was restricted to be unique.
The index sets contain the indices of the entries that survive the projections onto $C$, $I_j := \set{i\inint{1}{n} | r_i(j) \neq 0}$ for $j\inint{1}{h}$.
In other words, $p$ can be computed from $x$ by alternating projections, where the sets $L$ and $L_{I_j}$ are non-convex for all $j\inint{1}{h}$.
The expressions for the individual projections are given in Lemma~13, Lemma~17, Proposition~24, and Lemma~30, respectively, in \citet{Thom2013}.

Let $I_0 := \discint{1}{n}$ for completeness, then we can define the following constants for $j\inint{0}{h}$.
Let $d_j := \abs{I_j}$ be the number of relevant coordinates in each iteration, and let
\begin{displaymath}
  \beta_j := \sqrt{\frac{d_j\lambda_2^2 - \lambda_1^2}{d_j\snorm{V_{I_j}x}_2^2 - \snorm{V_{I_j}x}_1^2}}\text{ and }\alpha_j :=  \frac{1}{d_j}\left(\snorm{V_{I_j}x}_1 - \tfrac{\lambda_1}{\beta_j}\right)
\end{displaymath}
be real numbers.
We have that $d_j\lambda_2^2 - \lambda_1^2 \geq d_h\lambda_2^2 - \lambda_1^2 \geq 0$ by construction which implies $\beta_j > 0$ for all $j\inint{0}{h}$.
We now claim that the following holds:
\begin{enumerate}
  \item \label{thm:representation_a} $s_i(j) = \beta_j\cdot(x_i - \alpha_j)$ for all $i\in I_j$ for all $j\inint{0}{h}$.
  \item \label{thm:representation_b} $\alpha_0 \leq \cdots \leq \alpha_h$ and $\beta_0 \leq \cdots \leq \beta_h$.
  \item \label{thm:representation_c} $x_i\leq \alpha_j$ for all $i\not\in I_j$ for all $j\inint{0}{h}$.
  \item \label{thm:representation_d} $p = \beta_h\cdot\max\left(x - \alpha_h\cdot e,\ 0\right)$.
\end{enumerate}

We start by showing~\ref{thm:representation_a} with induction. For $j = 0$, we have $r(0) = x + \nicefrac{1}{n}\cdot\left(\lambda_1 - \snorm{x}_1\right)e$ using Lemma~13 from \citet{Thom2013}.
With Lemma~17 stated in \citet{Thom2013}, we have $s(0) = \delta r(0) + (1 - \delta) m$ with $m = \nicefrac{\lambda_1}{n}\cdot e$ and $\delta^2 = \left(\lambda_2^2 - \nicefrac{\lambda_1^2}{n}\right) \big/ \snorm{r(0) - m}_2^2$.
We see that $\snorm{r(0) - m}_2^2 = \snorm{x - \nicefrac{\norm{x}_1}{n}\cdot e}_2^2 = \snorm{x}_2^2 - \nicefrac{\snorm{x}_1^2}{n}$ and therefore $\delta = \beta_0$, and thus $s(0) = \beta_0\cdot\left(x - \nicefrac{1}{n}\cdot\left(\snorm{x}_1 - \nicefrac{\lambda_1}{\beta_0}\right)e\right)$, so the claim holds for the base case.

Suppose that~\ref{thm:representation_a} holds for $j$ and we want to show it also holds for $j + 1$.
It is $r(j+1) = \proj_C(s(j))$ by definition, and Proposition~31 in \citet{Thom2013} implies $r(j+1) = \max\left(s(j) - \hat{t}\cdot e,\ 0\right)$ where $\hat{t}\in\R$ can be expressed explicitly as $\hat{t} = \nicefrac{1}{d_{j+1}}\cdot\big(\sum_{i\in I_{j+1}} s_i(j) - \lambda_1\big)$, which is the mean value of the entries that survive the simplex projection up to an additive constant.
We note that $\hat{t}$ is here always non-negative, see Lemma~28(a) in \citet{Thom2013}, which we will need to show~\ref{thm:representation_b}.
Since $I_{j+1}\subsetneq I_j$ we yield $s_i(j) = \beta_j\cdot(x_i - \alpha_j)$ for all $i\in I_{j+1}$ with the induction hypothesis, and therefore we have that $\hat{t} = \nicefrac{1}{d_{j+1}}\cdot\big(\beta_j\snorm{V_{I_{j+1}}x}_1 - d_{j+1}\beta_j\alpha_j - \lambda_1\big)$.
We find that $r_i(j+1) > 0$ for $i\in I_{j+1}$ by definition, and we can omit the rectifier so that $r_i(j+1) = s_i(j) - \hat{t}$.
Using the induction hypothesis and the expression for $\hat{t}$ we have $r_i(j+1) = \beta_j x_i - \nicefrac{\beta_j}{d_{j+1}}\cdot\snorm{V_{I_{j+1}}x}_1 + \nicefrac{\lambda_1}{d_{j+1}}$.
For projecting onto $L_{I_{j+1}}$, the distance between $r(j+1)$ and $m_{I_{j + 1}} = \nicefrac{\lambda_1}{d_{j+1}}\cdot\sum_{i\in I_{j+1}}e_i$ is required for computation of $\delta^2 = \left(\lambda_2^2 - \nicefrac{\lambda_1^2}{d_{j+1}}\right) \big/ \snorm{r(j+1) - m_{I_{j + 1}}}_2^2$, so that Lemma~30 from \citet{Thom2013} can be applied.
We have that $\snorm{r(j+1) - m_{I_{j + 1}}}_2^2 = \sum_{i\in I_{j+1}} \big(\beta_j x_i - \nicefrac{\beta_j}{d_{j+1}}\cdot\snorm{V_{I_{j+1}}x}_1\big){}^2 = \beta_j^2\cdot\big(\snorm{V_{I_{j+1}}x}_2^2 - \nicefrac{\snorm{V_{I_{j+1}}x}_1^2}{d_{j+1}}\big)$, and further $\delta = \nicefrac{\beta_{j+1}}{\beta_j}$.
Now let $i\in I_{j+1}$ be an index, then we have $s_i(j+1) = \delta r_i(j+1) + (1 - \delta)\cdot\nicefrac{\lambda_1}{d_{j+1}} = \beta_{j+1}\cdot\big(x_i - \nicefrac{1}{d_{j+1}}\cdot\big(\snorm{V_{I_{j+1}}x}_1 - \nicefrac{\lambda_1}{\beta_{j+1}}\big)\big)$ using Lemma~30 from \citet{Thom2013}.
Therefore~\ref{thm:representation_a} holds for all $j\inint{0}{h}$.

Let us now turn to~\ref{thm:representation_b}.
From the last paragraph, we know that $\delta = \nicefrac{\beta_{j+1}}{\beta_j}$ for all $j\inint{0}{h-1}$ for the projections onto $L_{I_{j+1}}$.
On the other hand, we have that $\snorm{r(j+1) - m_{I_{j + 1}}}_2^2 = \snorm{r(j+1)}_2^2 - \nicefrac{\lambda_1^2}{d_{j+1}}$ from the proof of Lemma~30(a) from \citet{Thom2013}, and $\snorm{r(j+1)}_2^2 \leq\lambda_2^2$ holds from the proof of Lemma~28(f) in \citet{Thom2013}, so $\delta\geq 1$ which implies $\beta_0 \leq \cdots \leq \beta_h$.
As noted above, the separator for projecting onto $C$ satisfies $\hat{t}\geq 0$ for all $j\inint{0}{h-1}$.
By rearranging this inequality and using $\beta_j\leq \beta_{j+1}$, we conclude that $\alpha_j \leq \nicefrac{1}{d_{j+1}}\cdot\big(\snorm{V_{I_{j+1}}x}_1 - \nicefrac{\lambda_1}{\beta_j}\big) \leq \alpha_{j+1}$, hence $\alpha_0 \leq \cdots \leq \alpha_h$.

For~\ref{thm:representation_c} we want to show that the coordinates in the original vector $x$ which will vanish in some iteration when projecting onto $C$ are already small.
The base case $j = 0$ of an induction for $j$ is trivial since the complement of $I_0$ is empty.
In the induction step, we note that the complement of $I_{j+1}$ can be partitioned into $I_{j+1}\comp = I_j\comp\cup\big(I_j\cap I_{j+1}\comp\big)$ since $I_{j+1}\subsetneq I_j$.
For $i\in I_j\comp$ we already know that $x_i\leq \alpha_j\leq \alpha_{j+1}$ by the induction hypothesis and~\ref{thm:representation_b}.
We have shown in~\ref{thm:representation_a} that $s_i(j) = \beta_j\cdot(x_i - \alpha_j)$ for all $i\in I_j$, and if $i\in I_j\setminus I_{j+1}$ then $s_i(j) \leq\hat{t}$ since $0 = r_i(j+1) = \max\left(s_i(j) - \hat{t},\ 0\right)$.
By substituting the explicit expression for $\hat{t}$ and solving for $x_i$ we yield $x_i \leq \nicefrac{1}{d_{j+1}}\cdot\big(\snorm{V_{I_{j+1}}x}_1 - \nicefrac{\lambda_1}{\beta_j}\big) \leq \alpha_{j+1}$, and hence the claim holds for all $i\not\in I_{j+1}$.

If we can now show that~\ref{thm:representation_d} holds, then the claim of the theorem follows by setting $\alpha^* := \alpha_h$ and $\beta^* := \beta_h$.
We note that by construction $p = s(h)$ and $s_i(h) \geq 0$ for all coordinates $i\inint{1}{n}$.
When $i\in I_h$, then $s_i(h) = \beta_h\cdot(x_i - \alpha_h)$ with~\ref{thm:representation_a}, which is positive by requirement, so when the rectifier is applied nothing changes.
If $i\not\in I_h$ then $x_i - \alpha_h \leq 0$ by~\ref{thm:representation_c}, and indeed $\beta_h\cdot\max\left(x_i - \alpha_h,\ 0\right) = 0 = s_i(h)$.
The expression therefore holds for all $i\inint{1}{n}$, which completes the constructive proof of existence.

\emph{Existence (implicit):}
Existence of the projection is guaranteed by the {Weierstra\ss} extreme value theorem since $T$ is compact.
Now let $f\colon\R^n\to\R$, $s\mapsto \snorm{s - x}_2^2$, be the objective function, and let the constraints be represented by the functions  $h_1\colon\R^n\to\R$, $s\mapsto e\transp s - \lambda_1$, $h_2\colon\R^n\to\R$, $s\mapsto\snorm{s}_2^2 - \lambda_2^2$, and $g_i\colon\R^n\to\R$, $s\mapsto -s_i$, for all indices $i\inint{1}{n}$.
All these functions are continuously differentiable.
If $p = \proj_T(x)$, then $p$ is a local minimum of $f$ subject to $h_1(p) = 0$, $h_2(p) = 0$ and $g_1(p) \leq 0$, $\dots$, $g_n(p) \leq 0$.

For application of the method of Lagrange multipliers we first have to show that $p$ is regular, which means that the gradients of $h_1$, $h_2$ and $g_i$ for $i\not\in I$ evaluated in $p$ must be linearly independent \citep[Section~3.3.1]{Bertsekas1999}.
Let $J := I\comp = \discint{j_1}{j_{n - d}}$, say, then $\abs{J} \leq n - 2$ since $d\geq 2$.
Hence we have at most $n$ vectors from $\R^n$ for which we have to show linear independence.
Clearly $h_1'(s) = e$, $h_2'(s) = 2s$ and $g_i'(s) = -e_i$ for all $i\inint{1}{n}$.
Now let $u_1,u_2\in\R$ and $v\in\R^{n-d}$ with $u_1 e + 2u_2p - \sum_{\mu=1}^{n - d} v_\mu e_{j_\mu} = 0\in\R^n$.
Then, let $\mu\inint{1}{n-d}$, then $p_{j_\mu} = 0$ by definition of $I$ and hence by pre-multiplication of the equation above with $e_{j_\mu}\transp$ we yield $u_1 - v_\mu = 0\in\R$.
Therefore $u_1 = v_\mu$ for all $\mu\inint{1}{n-d}$.
On the other hand, if $i\in I$ then $p_i > 0$ and $e_i\transp e_{j_\mu} = 0$ for all $\mu\inint{1}{n-d}$.
Hence $u_1 + 2u_2 p_i = 0\in\R$ for all $i\in I$.
In the first paragraph of the uniqueness proof we have shown there are two distinct indices $i,j\in I$ with $p_i\neq p_j$.
Since $u_1 + 2u_2 p_i = 0 = u_1 + 2u_2 p_j$ and thus $0 = 2u_2(p_i - p_j)$ we can conclude that $u_2 = 0$, which implies $u_1 = 0$.
Then $v_1 = \dots = v_{n-d} = 0$ as well, which shows that $p$ is regular.

The Lagrangian is $\lag\colon\R^n\times\R\times\R\times\R^n\to\R$, $(s,\ \alpha,\ \beta,\ \gamma)\mapsto f(s) + \alpha h_1(s) + \beta h_2(s) + \sum_{i=1}^n \gamma_i g_i(s)$, and its derivative with respect to its first argument $s$ is given by $\lag'(s,\ \alpha,\ \beta,\ \gamma) := \nicefrac{\partial}{\partial s}\ \lag(s,\ \alpha,\ \beta,\ \gamma) = 2(s - x) + \alpha\cdot e + 2\beta\cdot s - \gamma\in\R^n$.
Now, Proposition~3.3.1 from \citet{Bertsekas1999} guarantees the existence of Lagrange multipliers $\tilde{\alpha},\tilde{\beta}\in \R$ and $\tilde{\gamma}\in\R^n$ with $\lag'(p,\ \tilde{\alpha},\ \tilde{\beta},\ \tilde{\gamma}) = 0$, $\tilde{\gamma}_i \geq 0$ for all $i\inint{1}{n}$ and $\tilde{\gamma}_i = 0$ for $i\in I$.
Assume $\tilde{\beta} = -1$, then $2x = \tilde{\alpha}\cdot e - \tilde{\gamma}$ since the derivative of $\lag$ must vanish.
Hence $x_i = \nicefrac{\tilde{\alpha}}{2}$ for all $i\in I$, and therefore $\set{p_i | i\in I}$ is a singleton with Remark~10 from \citet{Thom2013} as $p$ was assumed unique and $T$ is permutation-invariant. 
We have seen earlier that this is absurd, so $\tilde{\beta} \neq -1$ must hold.

Write $\alpha^* := \nicefrac{\tilde{\alpha}}{2}$, $\beta^* := \nicefrac{1}{(\tilde{\beta} + 1)}$ and $\gamma^* := \nicefrac{\tilde{\gamma}}{2}$ for notational convenience.
We then obtain $p = \beta^*(x - \alpha^*\cdot e + \gamma^*)$ because $\lag'$ vanishes.
Then $h_1(p) = 0$ implies that $\lambda_1 = \sum_{i\in I}p_i = \sum_{i\in I}\beta^*(x_i - \alpha^*) = \beta^*(\snorm{V_I x}_1 - d\alpha^*)$, and with $h_2(p) = 0$ follows that $\lambda_2^2 = \sum_{i\in I}p_i^2 = (\beta^*)^2\cdot (\snorm{V_I x}_2^2 - 2\alpha^*\snorm{V_I x}_1 + d\cdot (\alpha^*)^2)$.
By taking the ratio $\nicefrac{\lambda_1^2}{\lambda_2^2}$ and after elementary algebraic transformations we arrive at the quadratic equation
$a\cdot (\alpha^*)^2 + b\cdot \alpha^* + c = 0$, where $a := d\cdot\left(d - \nicefrac{\lambda_1^2}{\lambda_2^2}\right)$, $b := 2\snorm{V_I x}_1\cdot\left(\nicefrac{\lambda_1^2}{\lambda_2^2} - d\right)$ and $c := \snorm{V_I x}_1^2 - \snorm{V_I x}_2^2\cdot \nicefrac{\lambda_1^2}{\lambda_2^2}$ are reals.
The discriminant is $\Delta := b^2 - 4ac = 4\cdot \nicefrac{\lambda_1^2}{\lambda_2^2}\cdot\left(d - \nicefrac{\lambda_1^2}{\lambda_2^2}\right)\cdot\big(d\snorm{V_I x}_2^2 - \snorm{V_I x}_1^2\big)$.
Since $V_I x\in\R^d$ we have that $d\snorm{V_I x}_2^2 - \snorm{V_I x}_1^2 \geq 0$.
Moreover, the number $d$ is not arbitrary.
As $p$ exists by the {Weierstra\ss} extreme value theorem with $\snorm{p}_0 = d$, $\snorm{p}_1 = \lambda_1$ and $\snorm{p}_2 = \lambda_2$, we find that $\lambda_1 \leq \sqrt{d}\lambda_2$ and hence $\Delta \geq 0$, so there must be a real solution to the equation above.
Solving the equation shows that
\begin{displaymath}
  \alpha^*\in\Set{\frac{1}{d}\left(\snorm{V_I x}_1 \pm \lambda_1\sqrt{\frac{d\snorm{V_I x}_2^2 - \snorm{V_I x}_1^2}{d\lambda_2^2 - \lambda_1^2}}\right)}\text{,}
\end{displaymath}
hence from $h_1(p) = 0$ we obtain
\begin{displaymath}
  \beta^*\in\Set{\mp \sqrt{d\lambda_2^2 - \lambda_1^2} \Big/ \sqrt{d\snorm{V_I x}_2^2 - \snorm{V_I x}_1^2}}\text{.}
\end{displaymath}
Suppose $\alpha^*$ is the number that arises from the "$+$" before the square root, then $\beta^*$ is the number with the "$-$" sign, thus $\beta^* < 0$.
We have seen earlier that there are two distinct indices $i,j\in I$ with $p_i\neq p_j$.
We can assume $p_i > p_j$, then $0 < p_i - p_j = \beta^*(x_i - x_j)$ which implies that $x_i < x_j$.
This is not possible as it violates the order-preservation property of projections onto permutation-invariant sets \citep[Lemma~9(a) from][]{Thom2013}.
Thus our choice of $\alpha^*$ was not correct in the first place, and $\alpha^*$ must be as stated in the claim.

It remains to be shown that $p$ is the result of a soft-shrinkage function.
If $i\in I$, then $0 < p_i = \beta^*(x_i - \alpha^*)$, and $\beta^* > 0$ shows $x_i > \alpha^*$ such that $p_i = \beta^*\cdot\max\left(x_i - \alpha^*,\ 0\right)$.
When $i\not\in I$, we have $0 = p_i = \beta^*(x_i - \alpha^* + \gamma_i^*)$ where $\gamma_i^* \geq 0$ and still $\beta^* > 0$, thus $x_i \leq \alpha^*$ and $p_i = \beta^*\cdot\max\left(x_i - \alpha^*,\ 0\right)$ holds.
Therefore, the representation holds for all entries.
\qed
\end{proof}

Finding the set $I$ containing the indices of the positive coordinates of the projection result is the key for algorithmic computation of the projection.
Based on the constructive proof this could, for example, be achieved by carrying out alternating projections whose run-time complexity is between quasi-linear and quadratic in the problem dimensionality $n$ and whose space complexity is linear.
An alternative is the method proposed by \citet{Potluru2013}, where the input vector is sorted and then each possible candidate for $I$ is checked.
Due to the sorting, $I$ must be of the form $I = \discint{1}{d}$, where now only $d$ is unknown \citep[see also the proof of Theorem~3 from][]{Thom2013}.
Here, the run-time complexity is quasi-linear and the space complexity is linear in the problem dimensionality since also the sorting permutation has to be stored.
When $n$ gets large, algorithms with a smaller computational complexity are mandatory.

\subsection*{Properties of the Auxiliary Function}
We have already informally introduced the auxiliary function in Sect. \ref{sect:auxfunction}.
Here is a satisfactory definition:
\begin{definition}
Let $x\in\R_{\geq 0}^n\setminus T$ be a point such that $\proj_T(x)$ is unique and $\sigma(x) < \sigma^*$.
Let $x_{\max}$ denote the maximum entry of $x$, then
\begin{displaymath}
  \Psi\colon\intervalco{0}{x_{\max}}\to\R\text{,}\quad \alpha\mapsto\frac{\norm{\max\left(x - \alpha\cdot e,\ 0\right)}_1}{\norm{\max\left(x - \alpha\cdot e,\ 0\right)}_2} - \frac{\lambda_1}{\lambda_2}\text{,}
\end{displaymath}
is called \emph{auxiliary function} for the projection onto $T$.
\end{definition}

We call $\Psi$ \emph{well-defined} if all requirements from the definition are met.
Note that the situation where $\sigma(x) \geq \sigma^*$ is trivial, because in this sparseness-decreasing setup we have that all coordinates of the projection must be positive.
Hence $I = \discint{1}{n}$ in Theorem~\ref{thm:representation}, and $\alpha^*$ can be computed with the there provided formula.

We need more notation to describe the properties of $\Psi$.
Let $x\in\R^n$ be a point.
We write $\X := \set{x_i | i\inint{1}{n}\!}$ for the set that contains the entries of $x$.
Let $x_{\min} := \min\X$ be short for the smallest entry of $x$, and $x_{\max} := \max\X$ and $x_{\scndmax} := \max\X\setminus \set{x_{\max}}$ denote the two largest entries of $x$.
Further, $q\colon\R\to\R^n$, $\alpha\mapsto \max\left(x - \alpha\cdot e,\ 0\right)$, denotes the curve that evolves from application of the soft-shrinkage function to $x$.
The Manhattan norm and Euclidean norm of points from $q$ is given by $\ell_1\colon\R\to\R$, $\alpha\mapsto\norm{q(\alpha)}_1$, and $\ell_2\colon\R\to\R$, $\alpha\mapsto\norm{q(\alpha)}_2$, respectively.
Therefore, $\Psi = \nicefrac{\ell_1}{\ell_2} - \nicefrac{\lambda_1}{\lambda_2}$ and $\tilde{\Psi}$ from Sect.~\ref{sect:projfuncexperiments} can be written as $\nicefrac{\ell_1^2}{\ell_2^2} - \nicefrac{\lambda_1^2}{\lambda_2^2}$, such that its derivative can be expressed in terms of $\Psi'$ using the chain rule.

The next result provides statements on the auxiliary function's analytical nature and links its zero with the projection onto $T$:
\begin{lemma}
\label{lem:auxfunction}
Let $x\in\R_{\geq 0}^n\setminus T$ be given such that the auxiliary function $\Psi$ is well-defined.
Then the following holds:
\begin{enumerate}
  \item \label{lem:auxfunction_a} $\Psi$ is continuous on $\intervalco{0}{x_{\max}}$.
  \item \label{lem:auxfunction_b} $\Psi$ is differentiable on $\intervalco{0}{x_{\max}}\setminus\X$.
  \item \label{lem:auxfunction_c} $\Psi$ is strictly decreasing on $\intervalco{0}{x_{\scndmax}}$, and on $\intervalco{x_{\scndmax}}{x_{\max}}$ it is constant.
  \item \label{lem:auxfunction_d} There is exactly one $\alpha^*\in\intervaloo{0}{x_{\scndmax}}$ with $\Psi(\alpha^*) = 0$. It is then $\proj_T(x) = \left(\nicefrac{\lambda_2}{\ell_2(\alpha^*)}\right)\cdot q(\alpha^*)$.
\end{enumerate}
\end{lemma}

\begin{proof}
In addition to the original claim, we also give explicit expressions for the derivative of $\Psi$ and higher derivatives thereof in part~\ref{lem:auxfunction_c}.
These are necessary to show that $\Psi$ is strictly decreasing and constant, respectively, on the claimed intervals and for the explicit implementation of Algorithm~\ref{alg:auxiliary}.

\ref{lem:auxfunction_a}
The function $q$ is continuous because so is the soft-shrinkage function.
Hence $\ell_1$, $\ell_2$ and $\Psi$ are continuous as compositions of continuous functions.

\ref{lem:auxfunction_b}
The soft-shrinkage function is differentiable everywhere except at its offset.
Therefore, $\Psi$ is differentiable everywhere except for when its argument coincides with an entry of $x$, that is on $\intervalco{0}{x_{\max}}\setminus\X$.

\ref{lem:auxfunction_c}
We start with deducing the first and second derivative of $\Psi$.
Let $x_j,x_k\in\X\cup\set{0}$, $x_j < x_k$, such that there is no element from $\X$ between them.
We here allow $x_j = 0$ and $x_k = x_{\min}$ when $0\not\in\X$ for completeness.
Then the index set $I := \set{i\inint{1}{n} | x_i > \alpha}$ of non-vanishing coordinates in $q$ is constant for $\alpha\in\intervaloo{x_j}{x_k}$, and the derivative of $\Psi$ can be computed using a closed-form expression.
For this, let $d := \abs{I}$ denote the number of non-vanishing coordinates in $q$ on that interval.
With $\ell_1(\alpha) = \sum_{i\in I}\left(x_i - \alpha\right) = \sum_{i\in I}x_i - d\alpha$ we obtain $\ell_1'(\alpha) = -d$.
Analogously, it is $\nicefrac{\partial}{\partial\alpha}\ \ell_2(\alpha)^2 = \nicefrac{\partial}{\partial\alpha}\ \sum_{i\in I}\left(x_i - \alpha\right)^2 = -2\ell_1(\alpha)$, and the chain rule yields $\ell_2'(\alpha) = \frac{\partial}{\partial\alpha}\sqrt{\ell_2(\alpha)^2} = \nicefrac{-\ell_1(\alpha)}{\ell_2(\alpha)}$.
Application of the quotient rule gives $\Psi'(\alpha) = \left(\nicefrac{\ell_1(\alpha)^2}{\ell_2(\alpha)^2} - d\right) / \ell_2(\alpha)$.
The second derivative is of similar form.
We find $\nicefrac{\partial}{\partial\alpha}\ \ell_1(\alpha)^2 = -2d\ell_1(\alpha)$, and hence $\nicefrac{\partial}{\partial\alpha}\ \left(\nicefrac{\ell_1(\alpha)^2}{\ell_2(\alpha)^2}\right) = 2\left(\nicefrac{\ell_1(\alpha)}{\ell_2(\alpha)^2}\right)\cdot \left(\nicefrac{\ell_1(\alpha)^2}{\ell_2(\alpha)^2} - d\right)$.
We have $\nicefrac{\partial}{\partial\alpha}\ \left(\nicefrac{1}{\ell_2(\alpha)}\right) =  \nicefrac{\ell_1(\alpha)}{\ell_2(\alpha)^3}$ and with the product rule we see that $\Psi''(\alpha) = 3\left(\nicefrac{\ell_1(\alpha)}{\ell_2(\alpha)^3}\right)\cdot\left(\nicefrac{\ell_1(\alpha)^2}{\ell_2(\alpha)^2} - d\right) = 3\Psi'(\alpha)\cdot\nicefrac{\ell_1(\alpha)}{\ell_2(\alpha)^2}$.

First let $\alpha\in\intervaloo{x_{\scndmax}}{x_{\max}}$.
It is then $d = 1$, that is $q$ has exactly one non-vanishing coordinate.
In this situation we find $\ell_1(\alpha) = \ell_2(\alpha)$ and $\Psi'\equiv 0$ on $\intervaloo{x_{\scndmax}}{x_{\max}}$, thus $\Psi$ is constant on $\intervaloo{x_{\scndmax}}{x_{\max}}$ as a consequence of the mean value theorem from real analysis.
Because $\Psi$ is continuous, it is constant even on $\intervalco{x_{\scndmax}}{x_{\max}}$.

Next let $\alpha\in\intervalco{0}{x_{\scndmax}}\setminus\X$, and let $x_j$, $x_k$, $I$ and $d$ as in the first paragraph.
We have $d\geq 2$ since $\alpha < x_{\scndmax}$.
It is furthermore $\ell_1(\alpha) \leq \sqrt{d}\ell_2(\alpha)$ as $d = \snorm{q(\alpha)}_0$.
This inequality is in fact strict, because $q(\alpha)$ has at least two distinct nonzero entries.
Hence $\Psi'$ is negative on the interval $\intervaloo{x_j}{x_k}$, and the mean value theorem guarantees that $\Psi$ is strictly decreasing on this interval.
This property holds for the entire interval $\intervalco{0}{x_{\scndmax}}$ due to the continuity of $\Psi$.

\ref{lem:auxfunction_d}
The requirement $\sigma(x) < \sigma^*$ implies $\nicefrac{\snorm{x}_1}{\snorm{x}_2} > \nicefrac{\lambda_1}{\lambda_2}$ and thus $\Psi(0) > 0$.
Let $\alpha\in\intervaloo{x_{\scndmax}}{x_{\max}}$ be arbitrary, then $\ell_1(\alpha) = \ell_2(\alpha)$ as in~\ref{lem:auxfunction_c}, and hence $\Psi(\alpha) < 0$ since $\lambda_2 < \lambda_1$ must hold.
The existence of $\alpha^*\in\intervalco{0}{x_{\scndmax}}$ with $\Psi(\alpha^*) = 0$ then follows from the intermediate value theorem and~\ref{lem:auxfunction_c}.
Uniqueness of $\alpha^*$ is guaranteed because $\Psi$ is strictly monotone on the relevant interval.

Define $p := \proj_T(x)$, then with Theorem~\ref{thm:representation} there is an $\tilde{\alpha}\in\R$ so that $p = \left(\nicefrac{\lambda_2}{\snorm{\max\left(x - \tilde{\alpha}\cdot e,\ 0\right)}_2}\right)\cdot \max\left(x - \tilde{\alpha}\cdot e,\ 0\right) = \left(\nicefrac{\lambda_2}{\ell_2(\tilde{\alpha})}\right)\cdot q(\tilde{\alpha})$.
Since $p\in T$ we obtain $\Psi(\tilde{\alpha}) = 0$, and the uniqueness of the zero of $\Psi$ implies that $\alpha^* = \tilde{\alpha}$.
\qed
\end{proof}

As described in Sect.~\ref{sect:lintimeconstspacealg}, the exact value of the zero $\alpha^*$ of $\Psi$ can be found by inspecting the neighboring entries in $x$ of a candidate offset $\alpha$.
Let $x_j,x_k\in\X$ be these neighbors with $x_j\leq \alpha < x_k$ such that there is no element from $\X$ between $x_j$ and $x_k$.
When $\Psi$ changes its sign from $x_j$ to $x_k$, we know that its root must be located within this interval.
Further, we then know that all coordinates with a value greater than $x_j$ must survive the sparseness projection, which yields the set $I$ from Theorem~\ref{thm:representation} and thus the explicit representation of the projection.
The next result gathers these ideas and shows that it is easy to verify whether a change of sign in $\Psi$ is on hand.
\begin{lemma}
\label{lem:auxfunctionsignchange}
Let $x\in\R_{\geq 0}^n\setminus T$ be given such that $\Psi$ is well-defined and let $\alpha\in\intervalco{0}{x_{\max}}$ be arbitrary.
If $\alpha < x_{\min}$, let $x_j := 0$ and $x_k := x_{\min}$.
Otherwise, let $x_j := \max\set{x_i | x_i\in\X\text{ and }x_i\leq\alpha}$ be the left neighbor and let $x_k := \min\set{x_i | x_i\in\X\text{ and }x_i > \alpha}$ be the right neighbor of $\alpha$.
Both exist as the sets where the maximum and the minimum is taken are nonempty.
Define $I := \set{i\inint{1}{n} | x_i > \alpha}$ and $d := \abs{I}$.
Then:
\begin{enumerate}
  \item \label{lem:auxfunctionsignchange_a} When $\Psi(x_j)\geq 0$ and $\Psi(x_k) < 0$ then there is exactly one number $\alpha^*\in\intervalco{x_j}{x_k}$ with $\Psi(\alpha^*) = 0$.
  \item \label{lem:auxfunctionsignchange_b} It is $\ell_1(\xi) = \norm{V_I x}_1 - d\xi$ and $\ell_2^2(\xi) = \norm{V_I x}_2^2 - 2\xi\norm{V_I x}_1 + d\xi^2$ for $\xi\in\set{x_j,\ \alpha,\ x_k}$.
  \item \label{lem:auxfunctionsignchange_c} If the inequalities $\lambda_2\ell_1(x_j) \geq \lambda_1\ell_2(x_j)$ and $\lambda_2\ell_1(x_k) < \lambda_1\ell_2(x_k)$ are satisfied and $p := \proj_T(x)$ denotes the projection of $x$ onto $T$, then $I = \set{i\inint{1}{n} | p_i > 0}$ and hence $p$ can be computed exactly with Theorem~\ref{thm:representation}.
\end{enumerate}
\end{lemma}

\begin{proof}
The claim in~\ref{lem:auxfunctionsignchange_a} is obvious with Lemma~\ref{lem:auxfunction}.

\ref{lem:auxfunctionsignchange_b}
We find that $\ell_1(\alpha) = \sum_{i\in I}(x_i - \alpha) = \snorm{V_I x}_1 - d\alpha$ and $\ell_2(\alpha)^2 = \sum_{i\in I}(x_i - \alpha)^2 = \snorm{V_I x}_2^2 - 2\alpha\snorm{V_I x}_1 + d\alpha^2$.

We have $K = I \setminus \tilde{K}$ with $K := \set{i\inint{1}{n} | x_i > x_k}$ and $\tilde{K} := \set{i\inint{1}{n} | x_i = x_k}$.
One yields that $\ell_1(x_k) = \sum_{i\in K}(x_i - x_k) = \sum_{i\in I}(x_i - x_k) - \sum_{i\in\tilde{K}}(x_i - x_k) = \snorm{V_I x}_1 - d x_k$.
The claim for $\ell_2(x_k)^2$ follows analogously.

Likewise $I = J\setminus\tilde{J}$ with $J := \set{i\inint{1}{n} | x_i > x_j}$ and $\tilde{J} := \set{i\inint{1}{n} | x_i = x_j}$,
and hence follows $\ell_1(x_j) = \sum_{i\in J}(x_i - x_j) = \sum_{i\in I}(x_i - x_j) + \sum_{i\in\tilde{J}}(x_i - x_j) = \snorm{V_I x}_1 - d x_j$.
The value of $\ell_2(x_j)^2$ can be computed in the same manner.

\ref{lem:auxfunctionsignchange_c}
The condition in the claim is equivalent to the case of $\Psi(x_j) \geq 0$ and $\Psi(x_k) < 0$, hence with~\ref{lem:auxfunctionsignchange_a} there is a number $\alpha^*\in\intervalco{x_j}{x_k}$ with $\Psi(\alpha^*) = 0$.
Note that $\alpha\neq\alpha^*$ in general.
Write $p := \proj_T(x)$ and let $J := \set{i\inint{1}{n} | p_i > 0}$.
With Theorem~\ref{thm:representation} follows that $i\in J$ if and only if $x_i > \alpha^*$.
But this is equivalent to $x_i > x_j$, which in turn is equivalent to $x_i > \alpha$, therefore $I = J$ must hold.
Thus we already had the correct set of non-vanishing coordinates of the projection in the first place, and $\alpha^*$ and $\beta^*$ can be computed exactly using the formula from the claim of Theorem~\ref{thm:representation}, which yields the projection $p$.
\qed
\end{proof}

\subsection*{Proof of Correctness of Projection Algorithm}
In Sect.~\ref{sect:lintimeconstspacealg}, we informally described our proposed algorithm for carrying out sparseness-enforcing projections, and provided a simplified flowchart in Fig.~\ref{fig:projfunc-flowchart}.
After the previous theoretical considerations, we now propose and prove the correctness of a formalized algorithm for the projection problem.
Here, the overall method is split into an algorithm that evaluates the auxiliary function $\Psi$ and, based on its derivative, returns additional information required for finding its zero (Algorithm~\ref{alg:auxiliary}).
The other part, Algorithm~\ref{alg:projfunc}, implements the root-finding procedure and carries out the necessary computations to yield the result of the projection.
It furthermore returns information that will be required for computation of the projection's gradient, which we will discuss below.
\begin{theorem}
Let $x\in\R_{\geq 0}^n$ and $p := \proj_T(x)$ be unique.
Then Algorithm~\ref{alg:projfunc} computes $p$ in a number of operations linear in the problem dimensionality $n$ and with only constant additional space.
\end{theorem}

\begin{algorithm}[t]
  \caption{Linear time and constant space evaluation of the auxiliary function $\Psi$.}
  \label{alg:auxiliary}
  \SetAlgoLined
  \KwIn{Point to be projected $x\in\R_{\geq 0}^n$, target norms $\lambda_1, \lambda_2\in\R$, position $\alpha\in\intervalco{0}{x_{\max}}$ where $\Psi$ should be evaluated.}
  \KwOut{Values $\Psi(\alpha),\,\Psi'(\alpha),\,\Psi''(\alpha),\,\tilde{\Psi}(\alpha),\,\tilde{\Psi}'(\alpha)\in\R$, $\finished\in\booleanvalues$ indicating whether the correct interval has been found, numbers $\ell_1, \ell_2^2\in\R$ and $d\in\N$ needed to exactly compute $\alpha^*$ with $\Psi(\alpha^*) = 0$.}
  \BlankLine

  \tcp{Initialize.}
  $\ell_1 := 0$;\enspace $\ell_2^2 := 0$;\enspace $d := 0$\nllabel{algl:auxiliary-initialization}\;
  $x_j := 0$;\enspace $\Delta x_j := -\alpha$;\enspace $x_k := \infty$;\enspace $\Delta x_k := \infty$\;
  \BlankLine

  \tcp{Scan through $x$.}
  \For{$i := 1$ \KwTo $n$}
  {%
    $t := x_i - \alpha$\;
    \eIf{$t > 0$}
    {%
      $\ell_1 := \ell_1 + x_i$;\enspace $\ell_2^2 := \ell_2^2 + x_i^2$;\enspace $d := d + 1$\;
      \lIf{$t < \Delta x_k$}
      {%
        $x_k := x_i$;\enspace $\Delta x_k := t$%
      }
    }
    {%
      \lIf{$t > \Delta x_j$}
      {%
        $x_j := x_i$;\enspace $\Delta x_j := t$%
      }
    }
  }\nllabel{algl:auxiliary-scanfinished}
  \BlankLine

  \tcp{Compute $\Psi(\alpha)$, $\Psi'(\alpha)$ and $\Psi''(\alpha)$.}
  $\ell_1(\alpha) := \ell_1 - d\alpha$;\enspace $\ell_2(\alpha)^2 := \ell_2^2 - 2\alpha\ell_1 + d\alpha^2$\nllabel{algl:auxiliary-compute-psi}\;
  $\Psi(\alpha) := \nicefrac{\ell_1(\alpha)}{\sqrt{\ell_2(\alpha)^2}} - \nicefrac{\lambda_1}{\lambda_2}$\;
  $\Psi'(\alpha) := \left(\nicefrac{\ell_1(\alpha)^2}{\ell_2(\alpha)^2} - d\right) / \sqrt{\ell_2(\alpha)^2}$\;
  $\Psi''(\alpha) := 3\Psi'(\alpha)\cdot\nicefrac{\ell_1(\alpha)}{\ell_2(\alpha)^2}$\;
  \BlankLine

  \tcp{Compute $\tilde{\Psi}(\alpha)$ and $\tilde{\Psi}'(\alpha)$.}
  $\tilde{\Psi}(\alpha) := \nicefrac{\ell_1(\alpha)^2}{\ell_2(\alpha)^2} - \nicefrac{\lambda_1^2}{\lambda_2^2}$\;
  $\tilde{\Psi}'(\alpha) := \nicefrac{2\ell_1(\alpha)}{\sqrt{\ell_2(\alpha)^2}}\cdot\Psi'(\alpha)$\nllabel{algl:auxiliary-compute-psi-end}\;
  \BlankLine

  \tcp{Check for sign change from $\Psi(x_j)$ to $\Psi(x_k)$.}
  $\finished := \lambda_2(\ell_1 - dx_j) \geq \lambda_1\sqrt{\ell_2^2 - 2x_j\ell_1 + dx_j^2}$ {\bf and} $\phantom{\finished :=}\ \lambda_2\left(\ell_1 - dx_k\right) < \lambda_1\sqrt{\ell_2^2 - 2x_k\ell_1 + dx_k^2}$\nllabel{algl:auxiliary-sign-change-check}\;%
  \KwRet $\left(\Psi(\alpha),\Psi'(\alpha),\Psi''(\alpha),\tilde{\Psi}(\alpha),\tilde{\Psi}'(\alpha),\finished,\ell_1,\ell_2^2,d\right)$\;
\end{algorithm}

\begin{algorithm}[t]
  \caption{Linear time and constant space algorithm for projections onto $T$. The auxiliary function $\Psi$ is evaluated by calls to "$\aux$", which are carried out by Algorithm~\ref{alg:auxiliary}. This algorithm operates in-place, the input vector is overwritten by the output vector upon completion.}
  \label{alg:projfunc}
  \SetAlgoLined
  \SetKw{KwGoTo}{go to}
  \KwIn{Point to be projected $x\in\R_{\geq 0}^n$, target norms $\lambda_1, \lambda_2\in\R$, $\solver\in\set{\bisection,\ \newton,\ \newtonsqr,\ \halley}$.}
  \KwOut{Projection $\proj_T(x)\in T$ as first element replacing the input vector, numbers $\ell_1, \ell_2^2\in\R$ and $d\in\N$ needed to compute the projection's gradient with Algorithm~\ref{alg:projgrad}.}
  \BlankLine

  \tcp{Skip root-finding if decreasing sparseness.}
  $\left(\Psi(\alpha),\Psi'(\alpha),\Psi''(\alpha),\tilde{\Psi}(\alpha),\tilde{\Psi}'(\alpha),\finished,\ell_1,\ell_2^2,d\right) := \aux(x,\,\lambda_1,\,\lambda_2,\,0)$\;
  \lIf{$\Psi(\alpha) \leq 0$}{\KwGoTo Line~\ref{algl:projfunc-exact-alpha}}\nllabel{algl:projfunc-skiprootfind}
  \BlankLine

  \tcp{Sparseness should be increased.}
  $\lo := 0$;\enspace $\up := x_{\scndmax}$;\enspace $\alpha := \lo + \nicefrac{1}{2}\cdot (\up - \lo)$\;
  $\left(\Psi(\alpha),\Psi'(\alpha),\Psi''(\alpha),\tilde{\Psi}(\alpha),\tilde{\Psi}'(\alpha),\finished,\ell_1,\ell_2^2,d\right) := \aux(x,\,\lambda_1,\,\lambda_2,\,\alpha)$\;
  \BlankLine
  \tcp{Start root-finding procedure.}
  \While{{\bf not} $\finished$\nllabel{algl:projfunc-rootfind}}
  {%
    \tcp{Update bisection interval.}
    \leIf{$\Psi(a) > 0$}{$\lo := \alpha$}{$\up := \alpha$}
    \tcp{One iteration of root-finding.}
    \lIf{$\solver = \bisection$}{$\alpha := \lo + \nicefrac{1}{2}\cdot (\up - \lo)$}
    \Else(\CommentSty{// Use solvers based on derivatives.})
    {%
      \lIf{$\solver = \newton$}{$\alpha := \alpha - \nicefrac{\Psi(\alpha)}{\Psi'(\alpha)}$}
      \lElseIf{$\solver = \newtonsqr$}{$\alpha := \alpha - \nicefrac{\tilde{\Psi}(\alpha)}{\tilde{\Psi}'(\alpha)}$}
      \ElseIf{$\solver = \halley$}
      {%
        $h := 1 - \nicefrac{\left(\Psi(\alpha)\Psi''(\alpha)\right)}{\left(2\Psi'(\alpha)^2\right)}$\;
        $\alpha := \alpha - \Psi(\alpha) / \left(\max(0.5,\,\min(1.5,\,h))\cdot\Psi'(\alpha)\right)$\;
      }
      \tcp{Use bisection if $\alpha$ out of bounds.}
      \lIf{$\alpha < \lo$ {\bf or} $\alpha > \up$}{$\alpha := \lo + \nicefrac{1}{2}\cdot (\up - \lo)$}
    }
    \tcp{Evaluate auxiliary function anew.}
    $\left(\Psi(\alpha),\Psi'(\alpha),\Psi''(\alpha),\tilde{\Psi}(\alpha),\tilde{\Psi}'(\alpha),\finished,\ell_1,\ell_2^2,d\right) := \aux(x,\,\lambda_1,\,\lambda_2,\,\alpha)$\;
  }
  \BlankLine

  \tcp{Correct interval has been found, compute $\alpha^*$.}
  $\alpha^* := \dfrac{1}{d}\left(\ell_1 - \lambda_1\sqrt{d\ell_2^2 - \ell_1^2} \Big/ \sqrt{d\lambda_2^2 - \lambda_1^2}\right)$\nllabel{algl:projfunc-exact-alpha}\;
  \BlankLine

  \tcp{Compute result of the projection in-place.}
  $\rho := 0$\nllabel{algl:projfunc-comp-res}\;
  \For{$i := 1$ \KwTo $n$}
  {%
    $t := x_i - \alpha^*$;\enspace\leIf{$t > 0$}{$x_i := t$;\enspace $\rho := \rho + t^2$}{$x_i := 0$}
  }
  \lFor{$i := 1$ \KwTo $n$} {$x_i := \left(\nicefrac{\lambda_2}{\sqrt{\rho}}\right)\cdot x_i$}\nllabel{algl:projfunc-comp-beta}

  \KwRet $\left(x,\,\ell_1,\,\ell_2^2,\,d\right)$\;
\end{algorithm}

\begin{proof}
We start by analyzing Algorithm~\ref{alg:auxiliary}, which evaluates $\Psi$ at any given position $\alpha$.
The output includes the values of the auxiliary function, its first and second derivative, and the value of the transformed auxiliary function and its derivative.
There is moreover a Boolean value which indicates whether the interval with the sign change of $\Psi$ has been found, and three additional numbers required to compute the zero $\alpha^*$ of $\Psi$ as soon as the correct interval has been found.

Let $I := \set{i\inint{1}{n} | x_i > \alpha}$ denote the indices of all entries in $x$ larger than $\alpha$.
In the blocks from Line~\ref{algl:auxiliary-initialization} to Line~\ref{algl:auxiliary-scanfinished}, the algorithm scans through all the coordinates of $x$.
It identifies the elements of $I$, and computes numbers $\ell_1$, $\ell_2^2$, $d$, $x_j$ and $x_k$ on the fly.
After Line~\ref{algl:auxiliary-scanfinished}, we clearly have that $\ell_1 = \norm{V_I x}_1$, $\ell_2^2 = \norm{V_I x}_2^2$ and $d = \abs{I}$.
Additionally, $x_j$ and $x_k$ are the left and right neighbors, respectively, of $\alpha$.
Therefore, the requirements of Lemma~\ref{lem:auxfunctionsignchange} are satisfied.

The next two blocks spanning from Line~\ref{algl:auxiliary-compute-psi} to Line~\ref{algl:auxiliary-compute-psi-end} compute scalar numbers according to Lemma~\ref{lem:auxfunctionsignchange}\ref{lem:auxfunctionsignchange_b}, the definition of $\Psi$, the first two derivatives thereof given explicitly in the proof of Lemma~\ref{lem:auxfunction}\ref{lem:auxfunction_c}, the definition of $\tilde{\Psi}$ and its derivative given by the chain rule.
In Line~\ref{algl:auxiliary-sign-change-check}, it is checked whether the conditions from Lemma~\ref{lem:auxfunctionsignchange}\ref{lem:auxfunctionsignchange_c} hold using the statements from Lemma~\ref{lem:auxfunctionsignchange}\ref{lem:auxfunctionsignchange_b}.
The result is stored in the Boolean variable "$\finished$".

Finally all computed numbers are passed back for further processing.
Algorithm~\ref{alg:auxiliary} clearly needs time linear in $n$ and only constant additional space.

Algorithm~\ref{alg:projfunc} performs the actual projection in-place, and outputs values needed for the gradient of the projection.
It uses Algorithm~\ref{alg:auxiliary} as sub-program by calls to the function "$\aux$".
The algorithm first checks whether $\Psi(0) \leq 0$, which is fulfilled when $\sigma(x) \geq \sigma^*$.
In this case, all coordinates survive the projection, computation of $\alpha^*$ is straightforward with Theorem~\ref{thm:representation} using $I = \discint{1}{n}$.

Otherwise, Lemma~\ref{lem:auxfunction}\ref{lem:auxfunction_d} states that $\alpha^*\in\intervaloo{0}{x_{\scndmax}}$.
We can find $\alpha^*$ numerically with standard root-finding algorithms since $\Psi$ is continuous and strictly decreasing on the interval $\intervaloo{0}{x_{\scndmax}}$.
The concrete variant is chosen by the parameter "$\solver$" of Algorithm~\ref{alg:projfunc}, implementation details can be found in \citet{Traub1964} and \citet{Press2007}.

Here, the root-finding loop starting at Line~\ref{algl:projfunc-rootfind} is terminated once Algorithm~\ref{alg:auxiliary} indicates that the correct interval for exact computation of the zero $\alpha^*$ has been identified.
It is therefore not necessary to carry out root-finding until numerical convergence, it is enough to only come sufficiently close to $\alpha^*$.
Line~\ref{algl:projfunc-exact-alpha} computes $\alpha^*$ based on the projection representation given in Theorem~\ref{thm:representation}.
This line is either reached directly from Line~\ref{algl:projfunc-skiprootfind} if $\sigma(x) \geq \sigma^*$, or when the statements from Lemma~\ref{lem:auxfunctionsignchange}\ref{lem:auxfunctionsignchange_c} hold.
The block starting at Line~\ref{algl:projfunc-comp-res} computes $\max\left(x - \alpha^*\cdot e,\ 0\right)$ and stores this point's squared Euclidean norm in the variable $\rho$.
Line~\ref{algl:projfunc-comp-beta} computes the number $\beta^*$ from Theorem~\ref{thm:representation} and multiplies every entry of $\max\left(x - \alpha^*\cdot e,\ 0\right)$ with it, such that $x$ finally contains the projection onto $T$.
It would also be possible to create a new vector for the projection result and leave the input vector untouched, at the expense of additional memory requirements which are linear in the problem dimensionality.

When $\solver = \bisection$, the loop in Line~\ref{algl:projfunc-rootfind} is repeated a constant number of times regardless of $n$ (see Sect.~\ref{sect:lintimeconstspacealg}), and since Algorithm~\ref{alg:auxiliary} terminates in time linear in $n$, Algorithm~\ref{alg:projfunc} needs time only linear in $n$.
Further, the amount of additional memory needed is independent of $n$, as for Algorithm~\ref{alg:auxiliary}, such that the overall space requirements are constant.
Therefore, Algorithm~\ref{alg:projfunc} is asymptotically optimal in the sense of complexity theory.
\qed
\end{proof}

\subsection*{Gradient of the Projection}
\citet{Thom2013} have shown that the projection onto $T$ can be grasped as a function almost everywhere which is differentiable almost everywhere.
An explicit expression for the projection's gradient was derived, which depended on the number of alternating projections required for carrying out the projection.
Based on the characterization we gained through Theorem~\ref{thm:representation}, we can derive a much simpler expression for the gradient which is also more efficient to compute:
\begin{theorem}
\label{thm:projgrad}
Let $x\in\R_{\geq 0}^n\setminus T$ such that $p := \proj_T(x)$ is unique.
Let $\alpha^*,\beta^*\in\R$, $I\subseteq\discint{1}{n}$ and $d := \abs{I}$ be given as in Theorem~\ref{thm:representation}.
When $x_i\neq\alpha^*$ for all $i\inint{1}{n}$, then $\proj_T$ is differentiable in $x$ with $\nicefrac{\partial}{\partial x}\ \proj_T(x) = V_I\transp G V_I$, where the matrix $G\in\R^{d\times d}$ is given by
\begin{displaymath}
  G := \sqrt{\tfrac{b}{a}}E_d - \tfrac{1}{\sqrt{ab}}\left(\lambda_2^2\tilde{e}\tilde{e}\transp + d\tilde{p}\tilde{p}\transp - \lambda_1\left(\tilde{e}\tilde{p}\transp + \tilde{p}\tilde{e}\transp\right)\right)\text{,}
\end{displaymath}
with $a := d\norm{V_I x}_2^2 - \norm{V_I x}_1^2\in\R_{\geq 0}$ and $b := d\lambda_2^2 - \lambda_1^2\in\R_{\geq 0}$.
Here, $E_d\in\R^{d\times d}$ is the identity matrix, $\tilde{e} := V_I e\in\set{1}^d$ is the point where all coordinates are unity, and $\tilde{p} := V_I p\in\R_{>0}^d$.
\end{theorem}

\begin{proof}
When $x_i\neq\alpha^*$ for all $i\inint{1}{n}$, then $I$ and $d$ are invariant to local changes in $x$.
Therefore, $\alpha^*$, $\beta^*$ and $\proj_T$ are differentiable as composition of differentiable functions.
In the following, we derive the claimed expression of the gradient of $\proj_T$ in $x$.

Let $\tilde{x} := V_I x\in\R^d$, then $\alpha^* = \nicefrac{1}{d}\cdot\big(\snorm{\tilde{x}}_1 - \lambda_1\sqrt{\nicefrac{a}{b}}\big)$.
Define $\tilde{q} := V_I\cdot \max\left(x - \alpha^*\cdot e,\ 0\right)$, then $\tilde{q} = \tilde{x} - \alpha^*\cdot\tilde{e}\in\R_{> 0}^d$ because $x_i > \alpha^*$ for all $i\in I$.
Further $\tilde{p} = \lambda_2\cdot\nicefrac{\tilde{q}}{\snorm{\tilde{q}}_2}$, and we have $p = V_I\transp V_I p = V_I\transp\tilde{p}$ since $p_i = 0$ for all $i\not\in I$.
Application of the chain rule yields
\begin{displaymath}
  \frac{\partial p}{\partial x}
  = \frac{\partial V_I\transp\tilde{p}}{\partial\tilde{p}}\cdot
    \frac{\partial}{\partial\tilde{q}}\left(\lambda_2\cdot\frac{\tilde{q}}{\snorm{\tilde{q}}_2}\right)\cdot
    \frac{\partial\left(\tilde{x} - \alpha^*\cdot \tilde{e}\right)}{\partial\tilde{x}}\cdot
    \frac{\partial V_Ix}{\partial x}\text{,}
\end{displaymath}
and with $H := \lambda_2\cdot \left(\nicefrac{\partial}{\partial\tilde{q}}\ \left(\nicefrac{\tilde{q}}{\snorm{\tilde{q}}_2}\right)\right)\cdot\left( \nicefrac{\partial}{\partial\tilde{x}}\ \left(\tilde{x} - \alpha^*\cdot \tilde{e}\right)\right)\in\R^{d\times d}$ follows $\nicefrac{\partial p}{\partial x} = V_I\transp H V_I$,
thus it only remains to show $G = H$.

One obtains $\nicefrac{\partial}{\partial\tilde{q}}\ \left(\nicefrac{\tilde{q}}{\snorm{\tilde{q}}_2}\right) = \left(E_d - \nicefrac{\tilde{q}\tilde{q}\transp}{\snorm{\tilde{q}}_2^2}\right) / \snorm{\tilde{q}}_2$.
Since all the entries of $\tilde{q}$ and $\tilde{x}$ are positive, their $L_1$ norms equal the dot product with $\tilde{e}$.
We have $\snorm{\tilde{q}}_1 = \snorm{V_I x}_1 - d\alpha^* = \lambda_1\sqrt{\nicefrac{a}{b}}$, and we obtain that
$\snorm{\tilde{q}}_2^2 = \snorm{\tilde{x}}_2^2 - \alpha^*\cdot\big(\snorm{\tilde{x}}_1 + \lambda_1\sqrt{\nicefrac{a}{b}}\big) = \snorm{\tilde{x}}_2^2 - \nicefrac{1}{d}\cdot\big(\snorm{\tilde{x}}_1^2 - \lambda_1^2\cdot\nicefrac{a}{b}\big) = \nicefrac{a}{d}\cdot\left(1 + \nicefrac{\lambda_1^2}{b}\right) = \lambda_2^2\cdot\nicefrac{a}{b}$.
Now we can compute the gradient of $\alpha$.
Clearly $b$ is independent of $\tilde{x}$.
It is $\nicefrac{\partial a}{\partial\tilde{x}} = 2d\tilde{x}\transp - 2\snorm{\tilde{x}}_1\tilde{e}\transp\in\R^{1\times d}$.
Since $\tilde{x} = \tilde{q} + \alpha^*\cdot\tilde{e}$ it follows that $d\tilde{x} - \snorm{\tilde{x}}_1\tilde{e} = d\tilde{q} - \lambda_1\sqrt{\nicefrac{a}{b}}\cdot\tilde{e}$, and hence $\nicefrac{\partial \sqrt{a}}{\partial\tilde{x}} = \sqrt{\nicefrac{1}{a}}\cdot \big(d\tilde{q} - \lambda_1\sqrt{\nicefrac{a}{b}}\cdot\tilde{e}\big)\transp$.
Therefore, we conclude that $\nicefrac{\partial \alpha^*}{\partial\tilde{x}} = \left(\nicefrac{\lambda_2^2}{b}\right)\cdot\tilde{e}\transp - \left(\nicefrac{\lambda_1}{\sqrt{ab}}\right)\cdot\tilde{q}\transp\in\R^{1\times d}$.

By substitution into $H$ and multiplying out we see that
\begin{align*}
  H &= \sqrt{\tfrac{b}{a}}\left(E_d - \tfrac{b}{\lambda_2^2 a}\tilde{q}\tilde{q}\transp\right)
      \left(E_d - \tfrac{\lambda_2^2}{b}\tilde{e}\tilde{e}\transp + \tfrac{\lambda_1}{\sqrt{ab}}\tilde{e}\tilde{q}\transp\right)\\
  &= \sqrt{\tfrac{b}{a}}\Big(E_d - \tfrac{\lambda_2^2}{b}\tilde{e}\tilde{e}\transp + \tfrac{\lambda_1}{\sqrt{ab}}\tilde{e}\tilde{q}\transp\\
  &\phantom{= \sqrt{\tfrac{b}{a}}\Big(} -\tfrac{b}{\lambda_2^2 a}\tilde{q}\tilde{q}\transp + \tfrac{\lambda_1}{\sqrt{ab}}\tilde{q}\tilde{e}\transp - \tfrac{\lambda_1^2}{\lambda_2^2 a}\tilde{q}\tilde{q}\transp \Big)\text{,}
\end{align*}
where $\tilde{q}\tilde{q}\transp\tilde{e}\tilde{e}\transp = \tilde{q}\left(\tilde{q}\transp\tilde{e}\right)\tilde{e}\transp = \lambda_1\sqrt{\nicefrac{a}{b}}\cdot\tilde{q}\tilde{e}\transp$
and $\tilde{q}\tilde{q}\transp\tilde{e}\tilde{q}\transp = \lambda_1\sqrt{\nicefrac{a}{b}}\cdot\tilde{q}\tilde{q}\transp$ have been used.
The claim then follows with $\nicefrac{b}{\lambda_2^2 a} + \nicefrac{\lambda_1^2}{\lambda_2^2 a} = \nicefrac{d}{a}$ and $\tilde{q} = \sqrt{\nicefrac{a}{b}}\cdot\tilde{p}$.
\qed
\end{proof}

\begin{algorithm}[t]
  \caption{Product of the gradient of the projection onto $T$ with an arbitrary vector.}
  \label{alg:projgrad}
  \SetAlgoLined
  \SetKw{KwGoTo}{go to}
  \KwIn{A point $y\in\R^n$, target norms $\lambda_1, \lambda_2\in\R$, and the results of Algorithm~\ref{alg:projfunc}: Projection $p = \proj_T(x)$, numbers $\ell_1, \ell_2^2\in\R$ and $d\in\N$.}
  \KwOut{$z := \left(\nicefrac{\partial}{\partial x}\ \proj_T(x)\right)\cdot y\in \R^n$.}
  \BlankLine

  \tcp{Scan and slice input vectors.}
  $j := 0$;\enspace $\tilde{p}\in\set{0}^d$;\enspace $\tilde{y}\in\set{0}^d$;\enspace $\sy := 0$;\enspace $\scpy := 0$\;
  \For{$i := 1$ \KwTo $n$ {\bf where} $p_i > 0$}
  {%
    $j := j + 1$;\enspace $\tilde{p}_j := p_i$;\enspace $\tilde{y}_j := y_i$\;
    $\sy := \sy + \tilde{y}_j$;\enspace $\scpy := \scpy + \tilde{p}_j\cdot \tilde{y}_j$\;
  }\nllabel{algl:projgrad-init}
  \BlankLine

  \tcp{Compute gradient product in sliced space.}
  $a := d\ell_2^2 - \ell_1^2$;\enspace $b := d\lambda_2^2 - \lambda_1^2$\;
  $\tilde{y} := \sqrt{\nicefrac{b}{a}}\cdot\tilde{y}$\nllabel{algl:projgrad-compres}\;
  $\tilde{y} := \tilde{y} + \sqrt{\nicefrac{1}{ab}}\cdot\left(\lambda_1\cdot\sy - d\cdot\scpy\right)\cdot\tilde{p}$\;
  $\tilde{y} := \tilde{y} + \sqrt{\nicefrac{1}{ab}}\cdot\left(\lambda_1\cdot\scpy - \lambda_2^2\cdot\sy\right)\cdot\tilde{e}$\nllabel{algl:projgrad-compres-end}\;
  \BlankLine

  \tcp{Un-slice to yield final result.}
  $j := 0$;\enspace $z\in\set{0}^n$\;
  \lFor{$i := 1$ \KwTo $n$ {\bf where} $p_i > 0$}{$j := j + 1$;\enspace $z_i := \tilde{y}_j$}
  \KwRet $z$\;
\end{algorithm}

The gradient has a particular simple form, as it is essentially a scaled identity matrix with additive combination of scaled dyadic products of simple vectors.
In the situation where not the entire gradient but merely its product with an arbitrary vector is required (as for example in conjunction with the backpropagation algorithm), simple vector operations are already enough to compute the product:
\begin{corollary}
\label{cor:projgrad}
Algorithm~\ref{alg:projgrad} computes the product of the gradient of the sparseness projection with an arbitrary vector in time and space linear in the problem dimensionality $n$.
\end{corollary}
\begin{proof}
Let $y\in\R^n$ be an arbitrary vector in the situation of Theorem~\ref{thm:projgrad}, and define $\tilde{y} := V_I y\in\R^d$.
Then one obtains
\begin{align*}
  G\tilde{y} &= \sqrt{\tfrac{b}{a}}\tilde{y} + \tfrac{1}{\sqrt{ab}}\left(\lambda_1\cdot\tilde{e}\transp\tilde{y} - d\cdot\tilde{p}\transp\tilde{y} \right)\cdot\tilde{p}\\
  &\phantom{= \sqrt{\tfrac{b}{a}}\tilde{y}\;} + \tfrac{1}{\sqrt{ab}}\left(\lambda_1\cdot\tilde{p}\transp\tilde{y} - \lambda_2^2\cdot\tilde{e}\transp\tilde{y}\right)\cdot\tilde{e}\in\R^d\text{.}
\end{align*}
Algorithm~\ref{alg:projgrad} starts by computing the sliced vectors $\tilde{p}$ and $\tilde{y}$, and computes "$\sy$" which equals $\tilde{e}\transp\tilde{y}$ and "$\scpy$" which equals $\tilde{p}\transp\tilde{y}$ after Line~\ref{algl:projgrad-init}.
It then computes $a$ and $b$ using the numbers output by Algorithm~\ref{alg:projfunc}.
From Line~\ref{algl:projgrad-compres} to Line~\ref{algl:projgrad-compres-end}, the product $G\tilde{y}$ is computed in-place by scaling of $\tilde{y}$, adding a scaled version of $\tilde{p}$, and adding a scalar to each coordinate.
Since $\left(\nicefrac{\partial}{\partial x}\ \proj_T(x)\right)\cdot y = V_I\transp G\tilde{y}$, it just remains to invert the slicing.
The complexity of the algorithm is clearly linear, both in time and space.
\qed
\end{proof}

It has not escaped our notice that Corollary~\ref{cor:projgrad} can also be used to determine the eigensystem of the projection's gradient, which may prove useful for further analysis of gradient-based learning methods involving the sparseness-enforcing projection operator.

\bibliographystyle{plainnat}
\bibliography{the}
\end{document}